\documentclass[journal]{IEEEtran}
\IEEEoverridecommandlockouts
\usepackage{cite}
\usepackage{amsmath,amsthm,amssymb,amsfonts}
\usepackage{array}
\usepackage{graphicx}
\usepackage{textcomp}
\usepackage[mathscr]{eucal}
\usepackage[table]{xcolor}
\usepackage{xcolor}
\usepackage{cleveref}
\crefformat{section}{\S#2#1#3}
\usepackage[framemethod=tikz]{mdframed}
\usepackage[ruled,vlined,linesnumbered]{algorithm2e}
\usepackage{algpseudocode}
\usepackage{multirow}
\usepackage{rotating}
\newtheorem{definition}{Definition}
\newtheorem{lemma}{Lemma}

\usepackage[font=footnotesize]{caption}
\usepackage{subcaption}
\usepackage{pifont}
\usepackage{bbm}
\usepackage[flushmargin]{footmisc}

\algnewcommand\LeftComment[2]{%
\hspace{#1\algindent}$\triangleright$ \eqparbox{COMMENT}{#2} \hfill %
}
\let\oldnl\nl
\newcommand{\nonl}{\renewcommand{\nl}{\let\nl\oldnl}}
\def\BibTeX{{\rm B\kern-.05em{\sc i\kern-.025em b}\kern-.08em
    T\kern-.1667em\lower.7ex\hbox{E}\kern-.125emX}}
    
\begin{document}
\title{Designing and Training of Lightweight Neural Networks on Edge Devices using Early Halting in Knowledge Distillation}
\author{\IEEEauthorblockN{Rahul Mishra and Hari Prabhat Gupta}
 \thanks{The authors are with the Department of Computer Science and Engineering, Indian Institute of Technology (BHU) Varanasi, India (e-mail: rahulmishra.rs.cse17@iitbhu.ac.in; hariprabhat.cse@iitbhu.ac.in;)\\
 \noindent $\bullet$ A portion of this work was presented in ACM SenSys 2020~\cite{mishra2020teacher}.}}

\maketitle

\begin{abstract} 
Automated feature extraction capability and significant performance of Deep Neural Networks (DNN) make them suitable for Internet of Things (IoT) applications. However, deploying DNN on edge devices becomes prohibitive due to the colossal computation, energy, and storage requirements. This paper presents a novel approach for designing and training lightweight DNN using large-size DNN. The approach considers the available storage, processing speed, and maximum allowable processing time to execute the task on edge devices. We present a knowledge distillation based training procedure to train the lightweight DNN to achieve adequate accuracy. During the training of lightweight DNN, we introduce a  novel early halting technique, which preserves network resources; thus, speedups the training procedure.  Finally, we present the empirically and real-world evaluations to verify the effectiveness of the proposed approach under different constraints using various edge devices.
\end{abstract}


\section{Introduction}
Internet of Things (IoT) applications use sensors that generate a large amount of sensory data to perform a given task of real-time monitoring and detection~\cite{mishra2020teacher,8440758, 9130098}. In time-critical IoT applications such as fire or gas leakage detection in industrial warehouses, the sensory data processing must be completed within a specific time limit from its occurrence. Such time interval is referred to as Maximum Allowable Processing (MAP) time. Further, the edge devices are usually battery operated and smaller in size, having limited storage and processing capacity. Due to the limited storage and processing, such edge devices delayed the task execution in time-critical application~\cite{6384464,9599450,10.1145/3375877,8949724}. Therefore, it creates a vulnerable research challenge to execute a task within MAP time on a edge device.

Moreover, the high accuracy and automated feature extraction capability of Deep Neural Networks (DNN) make them suitable for IoT applications. However, the deployment of DNN on edge devices becomes prohibitive due to the excessive demand for resources~\cite{7994570}. Generally, the resources include storage and processing capacity. Next, a DNN compression technique transforms large-size DNN to lightweight for edge devices without significantly reducing performance~\cite{mishra2020survey}. A lightweight DNN has fewer parameters and can run on an edge device within limited storage. Further, lightweight DNN also reduces the inference time. Most of the existing compression techniques compress DNN up to a certain percentage without simultaneously considering the available resources of edge devices, desired accuracy, and MAP time of the task.

Knowledge Distillation (KD) is a concept that improves the performance of the lightweight DNN using the generalization ability of the large-size DNN~\cite{hinton2015distilling}. KD uses keywords teacher and student for large-size and lightweight DNN, respectively. It trains a student under the guidance of a teacher. Most of the existing KD approaches utilized the knowledge limited to the pre-trained teacher model and did not consider the knowledge from the training process of the teacher model. Different from the existing work, Zhao \textit{et al.}~\cite{9151346} employed the concept of using two teachers, \textit{i.e.,} scratch and pre-trained. Scratch teacher compels the student to follow an optimal path towards achieving final logits. A pre-trained teacher helps in avoiding the loss due to random initialization. The authors in~\cite{knowledge} proposed a framework, where a large-size DNN supervised the whole training process of lightweight DNN. The lightweight DNN shared parameters with large-size DNN to get low-level representation from the large-size. The main limitation of the existing work~\cite{knowledge,9151346} were not to considered the constraints of the edge devices while designing and training lightweight DNN. Furthermore, using multiple teachers~\cite{9151346} throughout the training of the student consumes huge resources and incurs colossal latency.

In this paper, we assume a given large-size DNN that can process a task successfully. However, it requires higher storage and processing time. Therefore, we propose an approach to design a lightweight DNN using a large-size DNN that can process the task in MAP time on edge devices. Next, to achieve higher accuracy using lightweight DNN, we present a knowledge distillation based lightweight DNN training scheme. The scheme introduces a novel early halting technique that significantly reduces training time and required resources. Specifically, we address the problem of designing and training a lightweight DNN using a given large-size DNN, where trained lightweight DNN satisfy the $\alpha$ and $\beta$ constraints of the edge devices. $\alpha$ and $\beta$ are the maximum available memory on edge devices and MAP time, respectively.

\noindent\textbf{Major contributions and novelty of the work:} \\
To the best of our knowledge, this is the first work to address the problem of designing and training lightweight DNN  by considering $\alpha$ and $\beta$ constraints of the edge devices. Along with this, the major contributions and novelty of this work are as follows:\\
\noindent $\bullet$ \textit{Transforming large-size to lightweight DNN}: The first contribution is to obtain a lightweight DNN from a large-size DNN. To do this, we dropout the unimportant units followed by reducing resource consumption from the large-size DNN using weight factorization and the minimal gated units on different layers of dropout DNN. The novel contributions lie in consideration of the number of connections in the given large-size DNN and maximum iteration runs for dropout. None of the existing work considers both in the dropout step. These novel considerations speed up the procedure of estimating the updated dropout rate. In addition, considering $\alpha$ and $\beta$ during dropout and reducing the resources (through weight factorization and minimal gated unit) makes our work different from existing work.

\noindent $\bullet$ \textit{Train the lightweight DNN}: We present a knowledge distillation based technique to train lightweight DNN (student) where, we incorporate two large-size DNN (teachers) with the same structural configuration, \textit{i.e.,} un-trained teacher and pre-trained teacher. We introduce a novel \textit{early halting technique}, where the student and the un-trained teacher are simultaneously trained up to certain (\textit{i.e.,} halting) epochs under the guidance of the pre-trained teacher. Afterwards, the student training is propagated under the guidance of the pre-trained teacher. Such a novel mechanism of early halting saves the resources; therefore, speedups the training procedure. The proposed training procedure transfers the knowledge from trained large-size DNN to lightweight DNN by minimizing the loss and improves its performance. Additionally, we propose an iterative algorithm to determine the optimal and trained lightweight model. Apart from neural architecture search~\cite{zoph2016neural}, the proposed algorithm required a limited number of steps due to $\alpha$ and $\beta$ constraints.

\noindent $\bullet$ \textit{Experimental validation}: We verify the effectiveness of the EarlyLight approach on the existing large-size DNN~\cite{9164991, xue2019deepfusion, yao2017deepsense, chen2019smartphone, janakiraman2018explaining, noori2020human}, public datasets, and edge devices. The results show that the proposed work can significantly improve performance and minimize latency. We also demonstrate real-world evaluation for locomotion mode recognition and evaluate the performance of the EarlyLight approach on parameters such as the model's size, training time, and different performance matrices.

The rest of paper is organized as follows. In the next section, we briefly discuss the literature on dropout, reducing resource requirements, and knowledge distillation to train the lightweight DNN. Section~\ref{prelim} presents the preliminary and overview of the solution for the problem addressed in this work. We propose an EarlyLight approach to train and design a lightweight DNN for edge devices in Section~\ref{propose_model}. The further two sections present the empirically and real-world evaluations. Finally, the paper concludes in Section~\ref{conclude_model}.

\vspace{-0.4cm}
\section{Background and motivation}
To better understand the motivation and background to design and train lightweight DNN for edge devices, we discuss the existing work emphasizing dropout, reducing resource requirements, and training the lightweight DNN using KD.\\
\noindent $\bullet$ \textit{Dropout in DNN:} The prior studies used random~\cite{srivastava2014dropout}, fixed~\cite{han2015learning,7837934}, or optimal~\cite{yao2017deepiot,10.5555/3305890.3305939} dropout methods for reducing resources of DNN. Authors in~\cite{srivastava2014dropout} highlighted the concept of random dropout to handle the overfitting problem in DNN. Such random dropout deteriorated the DNN structure. To mitigate the random dropout problem, Han \textit{et al.} in~\cite{han2015learning} proposed a mechanism of pruning and splicing side-by-side. The connection pruned during training can be spliced in back-propagation. They established a quadratic relation between the number of connections and neurons on the layers of DNN.  To obtain lightweight DNN, the authors in~\cite{7837934} disassembled a large DNN into small ones. They further estimated the gradients of smaller models. These gradients are compared to obtain the most reliable model. The fixed dropout~\cite{han2015learning,7837934} hampered the opportunities to improve the accuracy of the compressed DNN. Thus, the authors in~\cite{yao2017deepiot} proposed a DNN compression technique that has incorporated the estimation of optimal dropout rather than a fixed value. Further, the authors in~\cite{10.5555/3305890.3305939} exploited the concept of variational dropout only for fully connected and convolutional layers. However, they not considered recurrent layers of DNN.
 
\noindent $\bullet$ \textit{Reducing resource requirements of DNN:} The existing work reduced the resource requirements of DNN by reducing the complexity of computing units~\cite{bhattacharya2016sparsification, chauhan2018performance,gou2020knowledge, ofa, gordon2018morphnet, dai2019chamnet, yao2018fastdeepiot}, weights and biases~\cite{gordon2018morphnet, dai2019chamnet}, and filter pruning~\cite{luo2017thinet}. The authors in~\cite{bhattacharya2016sparsification} utilized the concept of layer factorization to reduce floating-point operations of fully connected layer and convolutional filter of DNN. The authors in~\cite{chauhan2018performance} performed DNN compression using weight quantization and layers pruning. The authors not considered the quantization scheme for convolutional and fully connected layers. Next, to reduce the massive resource demand and high complexity of neural architecture search. The authors in~\cite{ofa} proposed the concept of the once-for-all (OFA) network. OFA has facilitated one time operations to generate vast architectures with different specifications,  amortizing training cost. The authors decoupled training and architecture search stages with minimal accuracy compromise. The authors claimed to get a sub-network from OFA with no additional training cost. Similarly, the authors in~\cite{gordon2018morphnet} proposed the concept of iteratively shrinking and expanding DNN, utilizing sparsifying regularizer and uniform multiplicative factor, respectively. The authors named the concept as MorphNet. Apart from existing work on DNN compression, MorphNet has expanded (along with shrinkage) compressed DNN with increased available resources. MorphNet achieved performance improvement with smaller increment in training time. Another approach relying on the hardware traits for compression is presented by the authors in~\cite{dai2019chamnet}, named as ChamNet. The compressed DNN in ChamNet is obtained using an efficient evolutionary search, which takes baseline DNN, hardware traits, and energy availability as input.

\noindent $\bullet$ \textit{Training of lightweight DNN using KD:} Authors in~\cite{hinton2015distilling} proposed a KD technique, where the generalization ability of a pre-trained teacher is transferred to the student to improve its recognition performance. The logits of teacher and student are compared to estimate the distillation loss that should be minimized during the training of the student. The authors in~\cite{mishra2017apprentice} introduced the concept of simultaneous training of scratch teacher and student. It provided a soft target of logits for estimating the distillation loss between teacher and student. 
Next, Zhou \textit{et al.}~\cite{knowledge} presented a mechanism to share some initial layers of student and scratch teacher to improve the recognition accuracy. Further, the authors in~\cite{li2020few} utilized the KD technique to perform training of student using only a few samples of the dataset. The authors in~\cite{yang2020mobileda} presented a KD technique to handle domain disparity in testing data of teacher and student model. Finally, authors in~\cite{9151346} introduced the concept of pre-trained teacher and scratch teacher where, both teachers simultaneously guide student model. 

\noindent $\bullet$ \textbf{Motivation} This work is motivated by the following limitations, as noted in the existing literature. The prior work on the dropout technique in DNN~\cite{srivastava2014dropout, 7837934, han2015learning} used a fixed or random value of dropout. It leads to the pruning of important connections having lower weights, which results in significant accuracy compromise. Moreover, the work ~\cite{srivastava2014dropout,han2015learning,10.5555/3305890.3305939,babu2020single} do not guarantee the pruning of computing units in the recurrent neural network that consumes colossal resources. Next, the work in~\cite{bhattacharya2016sparsification, luo2017thinet, chauhan2018performance, yao2018fastdeepiot, lee2019neuro, liu2020layerwise} reduced the size of DNN. However, the authors did not consider the constraints for a given edge device (\textit{i.e.,} accuracy, execution time, and storage) while compressing the DNN. 

Further, to obtain a compressed (or lightweight) DNN using neural architecture search~\cite{elsken2019neural} is cost-ineffective, energy-consuming, and requires substantial resources for its execution. Therefore, it is required to develop a compression mechanism that preserves the time, energy, and resources during architecture search or training. Finally, the existing literature on knowledge distillation~\cite{hinton2015distilling, mishra2017apprentice, knowledge, chen2018distilling, li2020few, yang2020mobileda} adopted mechanisms to improve the performance of the lightweight DNN. However, none-of-the existing work emphasized reducing resources during the training of lightweight DNN and maintaining significant accuracy.
\vspace{-0.4cm}

\section{Preliminary and overview of solution}\label{prelim}
This section describes the terminologies and notations used in this work. We also discuss an overview of the solution to design a lightweight DNN from a large-size for a given edge device. Table~\ref{notation} illustrates the list of notations used in this work.

\subsection{Preliminary}
Let $\mathcal{D}$ denotes a dataset having $n$ instances and $k$ class labels, containing sensory measurements of $p$ different sensors. 
An instance $i$ of dataset $\mathcal{D}$ is denoted as $\mathbf{x}_i$, $\forall i \in \{1, \cdots, n\}$. Each instance $\mathbf{x}_i$ holds values of all $p$ sensors and corresponds to one class label $l$ of $k$ available classes, where, $l \in \{1, \cdots, k\}$. Let the large-size and lightweight DNN are denoted by $M^t$ and $M^s$, respectively.  

\begin{definition}[\textbf{Knowledge distillation}]
Knowledge distillation refers to a process for improving the performance of a lightweight DNN ($M^s$). Here, the knowledge (or generalization ability) of a large-size DNN ($M^t$) is utilized for training $M^s$, so the model $M^s$ can mimic a similar output pattern as $M^t$. This training from $M^t$ to $M^s$ is sometimes referred as \textbf{student-teacher training}~\cite{hinton2015distilling} in knowledge distillation.
\end{definition}

The training of student $M^s$ using knowledge distillation from teacher $M^t$ incorporates the comparison of their logits. The logits are the output features vector obtained at one layer before the softmax layer (output layer). Let $\mathbf{t}_i$ denote the logit vector of $M^t$ for $i^{th}$ training instance of dataset $\mathcal{D}$, where, $1\leq i \leq n$. Let $t_{ij}$ ($1\leq j \leq k$) is an element of $\mathbf{t}_{i}$, which can be estimated as $t_{ij}=w_{ij} x_{ij} + b_{j}$, where, $x_{ij}\in X$, $w_{ij}\in W^T $, and $b_{j}\in \mathbf{b}$ represent an element of feature matrix, weight matrix, and bias vector of teacher model, respectively. Similarly, we can estimate student logit vector $\mathbf{s}_i$ for $i^{th}$ training instance of $\mathcal{D}$. Next, we estimate distance between two vector using distance function $\delta(\mathbf{t}_i, \mathbf{s}_i)$ as: $\delta(\mathbf{t}_i, \mathbf{s}_i)=\|\mathbf{t}_i-\mathbf{s}_i\|_2^2$, where, $\|\cdot\|_2^2$ represents squared $l2$ norm. Further, the distance function $\delta(\mathbf{t},\mathbf{s})$ for all $n$ training instance in $\mathcal{D}$ is: $\delta(\mathbf{t},\mathbf{s})=\sum_{i=1}^n \|\mathbf{t}_i-\mathbf{s}_i\|_2^2$. The main objective of knowledge distillation is to minimize the distance function $\delta(\cdot)$ by training the student under the guidance of teacher for sufficient number of epochs.

\begin{definition}[\textbf{Logits}]
Logits refer to the feature vector generated by a DNN prior to the softmax layer. It is also termed as a non-normalized prediction vector of a DNN. These logits pass as input to the softmax layer for generating prediction probability against a testing instance. 
\end{definition}

\begin{definition}[\textbf{Maximum Allowable Processing time}]
A task in time-critical applications must be processed within a pre-defined time interval. Such time interval is known as Maximum Allowable Processing (MAP) time. The MAP time for a given task is denoted by $\beta$. Let an edge device processes $x$ FLOPs per unit time.  A task of $y$ FLOPs can successfully process on an edge device if $x \times y \leq \beta$.
\end{definition}

\subsection{Problem statement and overview of solution} 
Consider an edge device that can provide a maximum $\alpha$ space to store and process a task of $\beta$ MAP time. In this work, we investigate the following problem: \textit{how to design a lightweight DNN using a given large-size DNN  such that the trained lightweight DNN can successfully process a task on an edge device with given $\alpha$ and $\beta$ constraints?}

To design a lightweight DNN from a given large-size DNN, we propose the  EarlyLight approach that first designs a lightweight DNN for edge devices. The approach trains the lightweight DNN using the knowledge distillation technique. Section~\ref{new1} and Section~\ref{new2} present procedures to design a lightweight DNN from large-size using dropout and reducing the parameters of the computationally complex units, respectively. While designing lightweight DNN, we consider the given constraints $\alpha$ and $\beta$ of the edge device. We next present a procedure to train the designed lightweight DNN incorporating knowledge of pre-trained and un-training large-size DNN. We further introduce a novel early halting technique to accelerate the training of lightweight DNN while reducing the training resources and achieving adequate accuracy, discussed in Section~\ref{model_training}. Finally, we present an algorithm that uses all different procedures and considers large-size and lightweight DNN as input and output, respectively, to design and train lightweight DNN, \textit{i.e.}, processing the task within the constraints $\alpha$ and $\beta$ of edge device with high accuracy. 

 \begin{table}[h] 
	\caption{List of notations used in this work.}
	\centering
	\small
	\begin{tabular}{ >{\centering\arraybackslash}m{0.6cm} m{3.1cm} | >{\centering\arraybackslash}m{0.6cm} m{2.8cm}}
        
    \hline \textbf{\begin{tabular}[c]{@{}c@{}}Symbol \end{tabular}} & \multicolumn{1}{c}{\textbf{Description}} & \textbf{\begin{tabular}[c]{@{}c@{}}Symbol \end{tabular}} & \multicolumn{1}{c}{\textbf{Description}} \\ 
        
	\hline $\mathcal{D}$ & Dataset  & $n$ & Instances in $\mathcal{D}$\\
    $k$ & Number of classes in $\mathcal{D}$ & $W_i$ & Weight at layer $i$ \\
    $Q_i$ & Neurons at layer $i$ & $M^t$ & Teacher model\\
    $M^s$ & Student model & $\mathbf{\Pi}^s$ & Student classifier\\
    $\mathcal{L}_{DL}$ & Distillation loss & $\mathcal{L}_{CE}$ & Cross entropy loss \\
    $\mathcal{L}_{AL}$ & Attention loss & $\mathbf{x}_{te}$ & Testing instance \\
    $y_{te}$ & Testing label& $d$ & Dropout \\
    \hline
\end{tabular}
	\label{notation}
\end{table}

\section{EarlyLight: \underline{Light}weight neural networks on edge devices using \underline{Early} halting}\label{propose_model} 
This section proposes an approach to design and train lightweight DNN on edge devices using early halting in knowledge distillation, acronymed as \textit{EarlyLight}. The approach comprises mainly two phases: 1) designing of lightweight DNN for edge device and 2) training of the designed DNN. The designing phase involves the transformation of a given large-size DNN into a lightweight, considering the $\alpha$ and $\beta$ constraints of the edge device. We assume that a dataset $\mathcal{D}$ and a large-size DNN $M^t$ are given prior to this transformation. Later, the training phase introduces the technique of early hating in KD. The halting simultaneously reduces the training time and improves the accuracy of the designed lightweight DNN. Fig.~\ref{block} illustrates the overview of the EarlyLight approach.

\begin{figure}[h]
 \centering
 \includegraphics[scale=0.92]{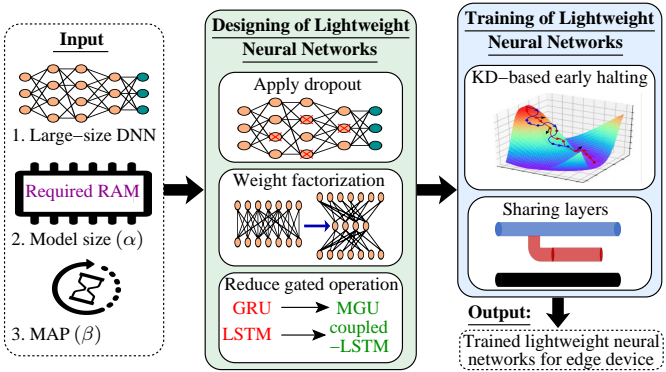}
 \caption{An overview of EarlyLight approach. MAP: Maximum Allowable Processing, GRU: Gated Recurrent Unit, MGU: Minimal Gated Unit, LSTM: Long Shot Term Memory, KD: Knowledge Distillation.}
 \label{block}
\end{figure}

\subsection{Designing of lightweight DNN}\label{phase1}
This section describes the technique of designing lightweight DNN for edge devices, satisfying $\alpha$ and $\beta$ constraints. We initially assume a large-size DNN ($M^{t}$) that is transformed into a lightweight DNN. First, we define the expression for execution time and memory consumption of lightweight DNN. Using the defined expressions, we deduce an optimization problem to minimize memory consumption and execution time for given constraints $\alpha$ and $\beta$, respectively. We next introduce the technique of estimating optimal dropout to reduce the resource requirement of $M^{t}$, which results in the dropout DNN. Later, the resources of the dropout DNN is minimized via weight factorization (convolutional and fully connected layers) and reduction in gated operations (recurrent layers). The resultant DNN is a lightweight neural network that satisfies the edge device's constraints $\alpha$ and $\beta$.      

Let $b_e$ and $e_m$ denote the memory and time requirements for executing single FLOP, respectively. Such $b_e$ and $e_m$ depend on the hardware capacity of the edge devices. We deduce the expression for temporary memory consumption ($T_{mem.}$) and execution time ($T_{exec.}$) to run the lightweight DNN on a given edge device. The expressions are given as:

\begin{align}\nonumber
 T_{mem.}=b_e\sum_{i=1}^{L}F_i,\quad T_{exec.}=e_m\sum_{i=1}^{L}F_i, 
\end{align}
where, $F_i$ denotes number of FLOPs for layer $i$ of regular large-size or reduced lightweight model, given in Table~\ref{fp1}. 

Finally, the objective function of a lightweight DDN $M^s$ for a given edge device with the average available space $\alpha$ and MAP time $\beta$ is given as:
\begin{align} \nonumber
 \min \text{  }  &\Omega T_{mem.} + (1- \Omega) T_{exec.}\\ \nonumber
 \textit{s.t., } 
 \mathbf{c_1:} \quad & \mathbb{T}_{mem.} \leq \alpha,\\ \label{opti}
  \mathbf{c_2:} \quad &\mathbb{T}_{exec.}\leq \beta,
\end{align}
where, $\Omega$ ($0 \leq \Omega \leq 1$) is used to neutralize the mismatch between units of execution time and memory consumption. Solving Eq.~\ref{opti} is tedious as the available resources on the edge devices changes dynamically. Therefore, we use a heuristic-based solution to apply dropout on large-size DNN and further reduce the resources requirement of dropout DNN. The resultant near-optimal lightweight DNN fits on the edge device, satisfying constraints in Eq.~\ref{opti}.

\subsubsection{Applying dropout on the large-size DNN}\label{new1}
We first apply the dropout over given large-size DNN ($M^{t}$) to curtail unimportant or inferior connections. The resultant dropout DNN is equivalent to a lightweight DNN with weights scaled with a given dropout rate. The dropout over $M^{t}$ reduces the required memory and execution time. Moreover, high and low dropout rates cause under-fitting and over-fitting of the DNN, respectively. A low dropout rate requires considerable resources with minimal or no accuracy compromise. However, high dropout rate leads to substantial accuracy compromise. This work estimates the optimal dropout that best suits our resources and accuracy requirements. To initialize the selection of optimal dropout, we set a dropout rate (denoted by $d$) preferably with a higher value like $d=0.5$ for hidden units and $d=0.8$ for input units~\cite{srivastava2014dropout}. Let $Q_b$ and $Q_a$ denote the number of connections, before and after dropout, respectively. Let $\max_{iteration}$, and $c$ are the maximum iteration runs for the dropout and a hyper-parameter, respectively. The updated dropout rate is given as follows: $d' \leftarrow d \times \max\left \{\sqrt{\frac{Q_b}{Q_a}},\left(1-\frac{iteration}{c \times \max_{iteration}}\right)\right \}$.  Additionally, when we consider the number of connections and maximum iterations of dropout apriori~\cite{babu2020single}. It speeds up the estimation of dropout rate. 

Furthermore, the initial value of $Q_b$ is the nothing but the number of connections in $M^t$. The $M^t$ can be represented as $\{W_i, Z_i:1\leq i \leq L\}$, where, $Z_i$ is a binary matrix that indicates the state of the network connection at layer $i$. It holds the information about a weight that retains or discarded on a given dropout. The binary matrix $Z_i$ is determined using discriminative function $f(\cdot)$ as,  $Z_i^{(j,k)}= f(W_i^{(j,k)}), \forall (j,k)\in \mathcal{I}$, where, $\mathcal{I}$ denotes the set of indices of $W_i$ at layer $i$. The function $f(\cdot)$ generates output $1$ if connection $Q_{i}^{j,k}$ remains after training and $0$ otherwise. The steps involved in the selection of optimal dropout is illustrated in Procedure~$1$. The weight $W$ in \textit{SGD\_function}() is updated using gradient descent with learning rate $\eta$. $\mathcal{L}(\cdot)$ is a cross-entropy loss associated with the DNN that captures discrepancy between predicted output and actual output.

\SetAlFnt{\small}
\begin{procedure}[h]
\label{proc1}
\caption{() \textbf{1: Applying dropout on large-size DNN.}}  
\KwIn{DNN $M^t$ $\{W_{i}, Z_i: 1\leq i \leq L \}$ with connection $Q_a$, learning rate $\eta$, loss of $M^t$ as $\mathcal{L}$, $l\in L$ layers upon which dropout is applied, hyper-parameter $c$ ;}
\KwOut{Dropout model with reduced connections $Q_b$\;}
\nonl Initialize $d^{(0)}\leftarrow 0.5$, $i\leftarrow 0$, $Q_a\leftarrow$ connections of $M^t$, $iteration\leftarrow 0$, $\max_{iteration}$\;
\SetKwRepeat{Do}{do}{while}
\Do{(${Loss}_i \le \mathcal{L}$)}{
Dropout $d^{(i)}$ on $l$ layers of $M^t$ with $Q_a$ connections\;
Estimate loss: $Loss_i \leftarrow$\textit{SGD\_function} ($W_{i}, Z_i$)\;
$Q_b \leftarrow$ reduced connections after dropout\;
$i \leftarrow i+1$\;
 Updating dropout using following formula:\\
$d^{(i)} \leftarrow d^{(i-1)} \times \max\left \{\sqrt{\frac{Q_b}{Q_a}},\left(1-\frac{iteration}{c \times \max_{iteration}}\right)\right \}$\;
$Q_a \leftarrow Q_b$ /*Updating connections*/\;
}
\smallskip
\Return Dropout model with connections $Q_b$\;
\smallskip
\nonl \textbf{Function} \textit{SGD\_function} ($W_i, Z_i$)\\
\nonl \hspace{10pt}\textbf{begin} \nonumber \\
\nonl \hspace{10pt} $\max_{iter} \leftarrow \frac{\|\mathcal{D}\|}{\|batch-size\|}$ , $iter$ $\leftarrow 1$\;
\nonl \hspace{20pt} \textbf{while} {$iter$ $\leq$ $\max_{iter}$} \textbf{do} \\
\nonl \hspace{30pt} Select a batch from training dataset $\mathcal{D}$ \;
\nonl \hspace{30pt} Perform forward propagation \;
\nonl \hspace{30pt} $\textit{Loss}\leftarrow$Estimate loss using $\mathcal{L}(W_i, X_i)$ \;
\nonl \hspace{30pt} Perform backward propagation and generate $\Delta \mathcal{L}$\;
\nonl \hspace{30pt} Update $W_i$ and $Z_i$ \;
\nonl \hspace{20pt}\Return ${Loss}$ \;
\nonl \hspace{10pt}\textbf{end}
\end{procedure}

\subsubsection{Reducing resources of dropout DNN}\label{new2}
Next, we describe the technique to reduce the resource requirements of the dropout DNN. This reduction enforces the designing of lightweight DNN that satisfies the $\alpha$ and $\beta$ constraints of the edge device. Apart from the prior work to reduce the resources of either convolutional or recurrent layers. This work introduces the technique to shrink the resource requirements of DNN layers, including convolutional, fully connected, and recurrent (Long Short Term Memory (LSTM) or Gated Recurrent Unit (GRU)). We apply weight factorization to reduce the resource requirements of the convolutional and fully connected layers.
Further, we eliminate the gates of the recurrent units to suppress the resources of LSTM and GRU. Procedure~$2$ summarizes the steps involved in reducing the dropout DNN, satisfying $\alpha$ and $\beta$ constraints of the edge devices.

Let $I_i$ and $O_i$ denote input and output dimensions of layer $i$ for dropout DNN, where $i$ may be Convolutional (Conv), Fully Connected (FC), LSTM, or GRU. The filter size, input channels, output channels of a convolutional layer $i$ is represented as $f_i \times g_i$, $h_i$, and $w_i$, respectively. Further, $s$, $L_{g}$, and $G_{g}$ denote step count, LSTM gates and GRU gates, respectively. The parameters ($P_i$) and required FLOPs ($F_i$) at layer $i$ of DNN are given in Table~\ref{fp1}(a).

\begin{table}[h]
\centering
\caption{Number of parameters and FLOPs at layer $i$.}
 \resizebox{.49\textwidth}{!}{
\begin{tabular}{|l@{}|c@{}|c|}
\multicolumn{3}{c}{(a) \textbf{\small Regular layer} }\\ \hline
\textbf{Layer}     &            \textbf{Parameter ($P_i$)}                                                     &                                \textbf{FLOPs (${F}_i$)}                                                           \\ \hline
Conv      & $(I_i\times (f_i \times g_i) \times O_i) + O_i$  &  $(f_i\times g_i) \times (I_i \times O_i) \times (h_i \times w_i)$        \\ \hline
FC        &       $(I_i \times O_i) + O_i$                   &    $(2I_i-1) \times O_i$                                                   \\ \hline
LSTM      &       $L_{g} O_i \times  (I_i + O_i +1)$                                                          &               $(2L_{g}O_i(I_i + O_i) +4O_i)s$                     \\ \hline
GRU & $G_{g} O_i\times (I_i + O_i +1)$ & $(2G_{g}O_i\times(I_i+ O_i) +5O_i)s$   \\ \hline 
\multicolumn{3}{c}{}\\ 
\multicolumn{3}{c}{(b) \textbf{\small After factorization and using minimal gated unit}}\\ \hline
Conv      &  $(I_i\times (f_i \times g_i) \times R_i) + R_i$ & $((f_i\times g_i) \times (h_i \times w_i)+1 + O_i)R_i$         \\ \hline
FC         &  $(I_i \times R_i) + R_i$                        & $((2I_i-1) + O_i) \times R_i$                                         \\ \hline
LSTM       &  $L_{g}{'} O_i \times  (I_i + O_i +1)$           &               $(2L_{g}{'}O_i^l \times(I_i + O_i) +4O_i)s$              \\ \hline
GRU        & $G_{g}{'} O_i\times (I_i + O_i +1)$              & $(2G_{g}{'}O_i\times(I_i+ O_i) +5O_i)s$                                \\ \hline
\end{tabular}
}
\label{fp1}
\end{table}

\textbf{(a)} \textit{Weight factorization:} 
This paper uses the weight factorization technique~\cite{bhattacharya2016sparsification} to reduce the parameters and FLOPs involved in convolutional and fully connected layers. The weight factorization technique introduces an intermediate multiplexing layer between two layers of the dropout DNN. This factorization of layers (convolutional or fully connected) reduces the computation requirement: if the size of the intermediate layer (denoted by $R_i$)  $\ngeq \frac{I_i\times O_i}{I_i+O_i}$~\cite{bhattacharya2016sparsification}. $R_i$ is obtained using a heuristic approach, where we start our factorization with $R_i<\frac{I_i\times O_i}{I_i+O_i}$ and estimate weight reconstruction error. Next, we iteratively decrease $R_i$ and estimate reconstruction error at each iteration. Finally, we obtain $R_i$ with minimum reconstruction error upon successful execution of this heuristic approach. The number of parameters and FLOPs at layer $i$ after the weight factorization are given in Table~\ref{fp1}(b).

\textbf{(b)} \textit{Reducing gated operations:} The parameters and FLOPs involved in the LSTM and GRU directly depend upon the gated operations. Therefore, we use the concept of MGU inspired from~\cite{zhou2016minimal} to reduce the resource requirements of DNN. MGU relies on the basic principle that the gated units play a significant role in achieving higher performance, whereas incorporating several gated operations increases computation complexity. Hence, a wiser selection of gates that persist in the network leads to comparable accuracy and low execution complexity. Let $L_{g}{'}$  and $G_{g}{'}$ denotes the reduced gates in LSTM and GRU, respectively. We replace LSTM with coupled LSTM and GRU with MGU to reduce the gated operations. 

\SetAlFnt{\small}
\begin{procedure}[ht]
\caption{() \textbf{2: Reducing resource of dropout DNN.}}
\label{proc2}
\KwIn{Dropout model with $Q_b$ connections\;}
\KwOut{Compressed model $M^s$\;}
Initialization: $i\leftarrow 1$, $j\leftarrow 0$, dropout applied on $l$ layers\;
\SetKwRepeat{Do}{do}{while}
\Do{(\textcolor{black}{Eq.~\ref{opti}} is not satisfied)}{
\For {$i,j \in \{1\leq i,j \leq l\}$ and $j\leftarrow i+1$\;}{
\If{$i$ = \text{Conv.} \textbf{or} $i$ = \text{FC} }{
Estimate reduced dimension $R_i$\;
Perform factorization by inserting a layer ${i}{'}$ (size $R_i$) between layers $i$ and ${j}$\;
Estimate parameters and FLOPs (Table~\ref{fp1}(b))\;
}
\If{$i$=\text{LSTM}}{
Replace the LSTM cells with coupled LSTM\;
Estimate parameters and FLOPs (Table~\ref{fp1}(b)) with updated gate $L_{g}{'}$ \;
}
\Else{
Replace the GRU cells with MGU\;
Estimate parameters and FLOPs (Table~\ref{fp1}(b)) with updated gate $G_{g}{'}$ \;
}
Solve optimization problem in \textcolor{black}{Eq.~\ref{opti}}\;
$i\leftarrow i+1$
}}
\Return $M^s$ with reduced parameters and FLOPs\;
\end{procedure}

\subsection{Training of lightweight DNN}\label{model_training}
This section covers the details about the training of lightweight DNN (or student) obtained in Section~\ref{phase1}. The student ($M^s$) is trained using the knowledge distillation technique, where we involve two teachers, \textit{i.e.,} a pre-trained large-size DNN ($M^{tr}$) and an un-trained large-size DNN ($M^{te}$). Further, we introduce the early halting technique for reducing the resource requirements for training $M^s$. Finally, this section derives the expression for optimal loss functions involved in the training. The main components of the training are: 1) knowledge distillation using early halting technique and 2) sharing layers of student and teacher. 

\subsubsection{Training $M^s$ using knowledge distillation with early halting}
Knowledge distillation from $M^{tr}$ to $M^{s}$, while training of $M^{s}$ on raw data improves its generalization ability. This improvement helps in enhancing the performance of $M^{s}$. In KD, the fine-tuned logits of $M^{tr}$ is compared against the logits of $M^{s}$, which are generated from raw data. Thus, the logits of $M^{tr}$ become a hard target for $M^{s}$. It also hinders the sufficient improvement in the performance of $M^{s}$. Thus, it could be beneficial to train an un-trained teacher ($M^{te}$) alongside $M^{s}$, where logits of $M^{te}$ is a soft target for $M^{s}$ . It provides soft-target during logits comparison. $M^{tr}$ and $M^{te}$ have same structural configuration. However, the un-trained teacher may sometimes undergo wrong random initialization, which leads to performance deterioration of $M^{s}$. 

Moreover, if we use both $M^{tr}$ and $M^{te}$ during training of $M^{s}$ then the problems, \textit{i.e.,} hard logits target and performance diminution due to random initialization is solved~\cite{9151346}. It also leads to significant improvement in the performance of $M^{s}$. Despite the successful training of $M^{s}$ due to appropriate matching of student and teachers logits, the simultaneous consideration of $M^{s}$, $M^{tr}$ and $M^{te}$ during training of $M^{s}$ demands colossal resources. As it requires three models to be trained, \textit{i.e.,} $M^{tr}$ followed by $M^{te}$ and $M^{s}$. 

We introduce the technique for \textit{early halting} of $M^{te}$ training after halting epoch $h$, where $h < E$ and $E$ denotes the total required epochs for the training, to reduce the resource consumption during training of $M^{s}$. The early halting saves the device's resources during training of $M^{s}$ and therefore fasten the training. Hereafter, the training of $M^{s}$ continues only under the guidance of trained $M^{tr}$, as shown in Fig.~\ref{model}(a). The early halting technique uses cross-entropy loss $\mathcal{L}_{CE}(\cdot)$, attention loss $\mathcal{L}_{AL}(\cdot)$, and distillation loss $\mathcal{L}_{DL}(\cdot)$, as shown in Fig.~\ref{mgu1}. The performance of $M^{s}$ can be improved in the supervision of trained $M^{tr}$ that compares output at each epoch. The comparison is carried out using attention loss between $M^{s}$ and $M^{tr}$. The combined loss ($\mathcal{L}_{comb}(\cdot)$), which operates during training of $M^{s}$ and $M^{te}$, is given as follows: 

\begin{small}
\begin{align}\label{comb1}
\mathcal{L}_{comb}(\cdot) =\left\{\begin{matrix}
 \lambda_1 \mathcal{L}^{s}_{CE}(\cdot)  + \lambda_2 \mathcal{L}_{AL}(\cdot) 
 +\lambda_3 \mathcal{L}_{DL}(\cdot) + \lambda_4\mathcal{L}^{te}_{CE}(\cdot), 
 \\\quad \text{till training of un-trained $M_i$},\\
\lambda_1 \mathcal{L}_{CE}^s(\cdot) + \lambda_2 \mathcal{L}_{AL}(\cdot) + \lambda_3 \mathcal{L}_{DL}(\cdot). 
\end{matrix}\right.
\end{align} 
\end{small} 
where $\lambda_1$, $\lambda_2$, $\lambda_3$, and $\lambda_4$ are the fractional contribution of different loss functions, $0 < \{ \lambda_1, \lambda_2, \lambda_3, \lambda_4\} \le 1$. We only optimize the combined loss associated with $M^{s}$, as the contribution of the loss of untrained $M^{te}$ is uniform throughout the training of $M^{s}$. The early halting optimizes the following problem:
\begin{subequations}\label{Main_optimization}
\begin{align}
 \min\text{  } & \mathcal{L}_{comb}^{s}(\cdot)\\
 \textit{s.t.,} \hspace{.3cm} &\hspace{.1cm}  \lambda_1 + \lambda_2 + \lambda_3 =1,\\
  &\hspace{.1cm} 0 < \{ \lambda_1, \lambda_2, \lambda_3\} <1.
\end{align}
\end{subequations}

\begin{lemma}\label{lemma}
The optimization problem in Eq.~\ref{Main_optimization} holds a convex optimal solution.  
\end{lemma}
\begin{proof}
 \noindent We determine the first order and second order derivative of $\mathcal{L}_{comb}^s(\cdot)$ to prove convexity of the optimization problem given in Eq.~\ref{Main_optimization}. Here, if second-order derivative of $\mathcal{L}_{comb}^s(\cdot)$ is positive, we can conclude the optimization problem is convex.
 
\begin{align}\label{gradient}
\hspace{-0.2cm}\frac{\mathrm{d} \mathcal{L}_{comb}^s(\cdot)}{\mathrm{d} x_{ij}} = \lambda_1\underset{\mathbf{T_1}}{\underbrace{\frac{\mathrm{d} \mathcal{L}_{CE}^s(\cdot)}{\mathrm{d} x_{ij}}}} + \lambda_2\underset{\mathbf{T_2}}{\underbrace{\frac{\mathrm{d} \mathcal{L}_{AL}(\cdot)}{\mathrm{d} x_{ij}}}} + \lambda_3\underset{\mathbf{T_3}}{\underbrace{\frac{\mathrm{d} \mathcal{L}_{DL}(\cdot)}{\mathrm{d} x_{ij}}}}.
\end{align}

\noindent $\textbf{(a)}$ Estimating the derivative of terms $\mathbf{T_1}$:
\begin{small}
\begin{align}\nonumber
\hspace{-0.2cm}\frac{\mathrm{d} \mathcal{L}_{CE}^s(\cdot)}{\mathrm{d} x_{ij}} & = -\frac{1}{n} \frac{\mathrm{d}}{\mathrm{d} x_{ij}}\sum_{i=1}^{n}\sum_{j=1}^{k}\mathbbm{1}(\cdot)  \log \frac{e^{(w_{ij}x_{ij} + b_j)}}{\sum_{j=1}^{k}e^{(w_{ij}x_{ij} + b_j)}},\\  \label{fce}
&= -\frac{\pmb{1}(\cdot)w_{ij}}{n} \times \Big\{ 1-\frac{e^{(w_{ij}x_{ij} + b_j)}}{\sum_{j=1}^{k}e^{(w_{ij}x_{ij} + b_j)}} \Big\}.\\ \label{sce}
 \hspace{-0.3cm}\frac{\mathrm{d}^2 \mathcal{L}_{CE}^s(\cdot)}{\mathrm{d}^2 x_{ij}} &=  \frac{\mathbbm{1}(\cdot)w_{ij}^2e^{(w_{ij}x_{ij} + b_j)}}{n\sum_{j=1}^{k}e^{(w_{ij}x_{ij} + b_j)}} \Big( 1-\frac{e^{(w_{ij}x_{ij} + b_j) }}{\sum_{j=1}^{k}e^{(w_{ij}x_{ij} + b_j)}} \Big).
\end{align}
\end{small}

\noindent $\textbf{(b)}$ Determining the derivative of terms $\mathbf{T_2}$:
\begin{small}
\begin{align}\nonumber
\frac{\mathrm{d} \mathcal{L}_{AL}(\cdot)}{\mathrm{d} x_{ij}} &= \frac{1}{n}\frac{\mathrm{d}}{\mathrm{d} x_{ij}}\sum_{i=1}^{n}\sum_{j=1}^{k}\Big\| \frac{\mathscr{T}^t(\mathcal{F}^t_{ij})}{||\mathscr{T}^t(\mathcal{F}^t_{ij})||} - \frac{\mathscr{T}^s(\mathcal{F}^s_{ij})}{||\mathscr{T}^s(\mathcal{F}^s_{ij})||} \Big \|_2^2.
\end{align}
\end{small}

We consider the value of transformation function and magnitude of teacher as constant because these values do not vary during student training. Similarly, magnitude of student transformation $||\mathscr{T}^s(\mathcal{F}^s_{ij})||$ is constant. Further, $\mathcal{F}^s_{ij}=w_{ij}x_{ij}+b_j$, therefore, we obtain first order derivative, as

\begin{small}
\begin{align}\label{fdl}
\hspace{-0.3cm}\frac{\mathrm{d} \mathcal{L}_{AL}(\cdot)}{\mathrm{d} x_{ij}}&=\frac{-2}{n}\Big(\frac{\mathscr{T}^t(\mathcal{F}^t_{ij})}{||\mathscr{T}^t(\mathcal{F}^t_{ij})||} - \frac{\mathscr{T}^s(\mathcal{F}^s_{ij})}{||\mathscr{T}^s(\mathcal{F}^s_{ij})||}\Big)\frac{w_{ij}}{||\mathscr{T}^s(\mathcal{F}^s_{ij})||}.\\ \label{sdl}
\frac{\mathrm{d}^2 \mathcal{L}_{AL}(\cdot)}{\mathrm{d}^2 x_{ij}}&=\frac{2}{n}\Big(\frac{w_{ij}}{||\mathscr{T}^s(\mathcal{F}^s_{ij})||}\Big)^2.
\end{align}
\end{small}

\noindent $\textbf{(c)}$ Estimating the derivative of terms $\mathbf{T_3}$: Similar to that of consideration for term $\mathbf{T_2}$, here also we assume that the part  of trainee is constant with respect to student.

\begin{small}
\begin{align}
 \frac{\mathrm{d} \mathcal{L}_{DL}(\cdot)}{\mathrm{d} x_{ij}} &= \frac{2}{n}(\mathbf{s}-\mathbf{z})w_{ij},\\\label{sds}
 \frac{\mathrm{d}^2 \mathcal{L}_{DL}(\cdot)}{\mathrm{d}^2 x_{ij}} &= \frac{2}{n}(w_{ij})^2.
\end{align}
\end{small}

From Eq.~\ref{sce}, Eq.~\ref{sdl}, and Eq.~\ref{sds}, we can obtain second order derivative of $\mathcal{L}_{comb}^s(\cdot)$ as

\begin{small}
\begin{align}
\hspace{-0.2cm}\frac{\mathrm{d}^2 \mathcal{L}_{comb}^s(\cdot)}{\mathrm{d}^2 x_{ij}} = \lambda_1\frac{\mathrm{d}^2 \mathcal{L}_{CE}(\cdot)}{\mathrm{d}^2 x_{ij}} + \lambda_2\frac{\mathrm{d}^2 \mathcal{L}_{AL}(\cdot)}{\mathrm{d}^2 x_{ij}} + \lambda_3\frac{\mathrm{d}^2 \mathcal{L}_{DL}(\cdot)}{\mathrm{d}^2 x_{ij}}.
\end{align}
\end{small}

As $\frac{\mathrm{d}^2 \mathcal{L}_{CE}(\cdot)}{\mathrm{d}^2 x_{ij}}>0$, $\frac{\mathrm{d}^2 \mathcal{L}_{AL}(\cdot)}{\mathrm{d}^2 x_{ij}}>0$, and $\frac{\mathrm{d}^2 \mathcal{L}_{DL}(\cdot)}{\mathrm{d}^2 x_{ij}}>0$, therefore, $\frac{\mathrm{d}^2 \mathcal{L}_{comb}^s(\cdot)}{\mathrm{d}^2 x_{ij}}>0$, 
which proves the convexity of combined loss $\mathcal{L}_{comb}^s(\cdot)$. 
\end{proof}
 
\begin{figure}[h]
    \vspace{-0.3cm}
     \centering
    \includegraphics[scale=0.90]{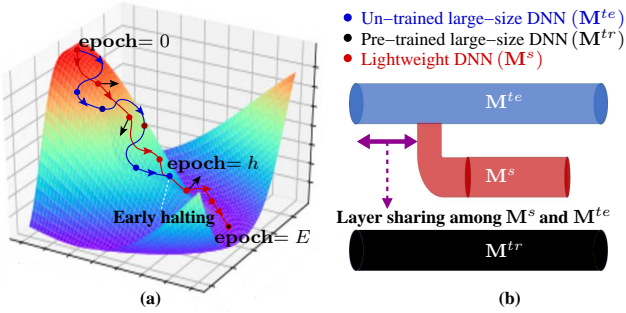}
     \caption{Training of lightweight DNN ($M^{s}$) : (a) early halting of $M^{te}$ training and (b) layer sharing among $M^s$ and $M^{te}$.}
     \label{model}
     \vspace{-0.3cm}
 \end{figure}
  
\subsubsection{Sharing layers of teacher and student}
Inspired by the concept of layer sharing among student and teacher, as discussed in~\cite{knowledge}; in this work, we share the first $i$ layers of $M^{s}$ and $M^{te}$. In other words, first, $i$ layers of $M^s$ and $M^{te}$ are the same, as shown in Fig.~\ref{mgu1}. Layer sharing can be better visualized using Fig.~\ref{model}(b), where, $M^{s}$ is derived from $M^{te}$. The layer sharing improves the performance of the trained $M^s$ and provides less variation in its output predicted probabilities. The layer sharing also preserve resources during training of $M^s$. In this work, we use sensory data for DNN training; thus, the complexity of large-size DNN is low. This lower complexity helps in obtaining lightweight DNN with minimal compression of large-size DNN. Through experimental analysis, we obtained that even at $50\%$ layer sharing, the resource constraints of edge devices are satisfied. Thus, we are using $50\%$ common layers of $M^{s}$ and $M^{te}$.

\begin{figure}[h]
 \vspace{-0.3cm}
 \centering
 \includegraphics[scale=0.75]{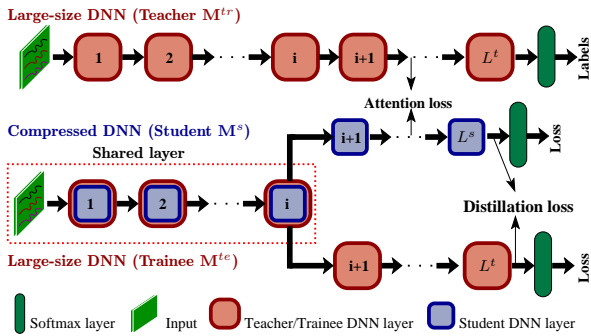}
 \caption{Training of $M^{te}$ and $M^{s}$ under guidance of $M^{tr}$.}
 \label{mgu1}
 \vspace{-0.7cm}
\end{figure}

\subsection{EarlyLight algorithm}
Algorithm~\ref{joint_algo} illustrates different steps involved in the EarlyLight approach to design and train lightweight DNN satisfying 
$\alpha$ and $\beta$ constraints of edge devices. The algorithm (Algorithm~\ref{joint_algo}) uses Procedure $1$ and Procedure $2$ to transform a large-size DNN into a lightweight DNN. It starts with the application of dropout on large-size DNN (Procedure $1$) followed by reducing resources of dropout DNN using Procedure $2$. These procedures are repeated until the constraints $\alpha$ and $\beta$ are not satisfied. The lightweight DNN obtained from Procedure $2$ is trained using knowledge distillation with early halting technique. The designed and trained DNN satisfy not only $\alpha$ and $\beta$ but also achieves adequate performance.

\begin{algorithm}[h]
\caption{\textbf{EarlyLight algorithm.}} 
\label{joint_algo}
\KwIn{Dataset ${D}$, teacher $M^{tr}$, trainee $M^{te}$, available space $\alpha$, MAP time $\beta$, halting epoch $h$, training epoch $E$\;}
\KwOut{Optimal lightweight model $M^{s}$\;}
\smallskip
Select a large-size DNN with $Q_a$ connections and $L$ layers\;
Identify $L{'}$ layers that are not to be shared\;
\smallskip
\nonl /* Variables initialization*/\\
\smallskip
$\delta_{L}\leftarrow 0$, $l \leftarrow L{'}/2$; /*$50\%$ of non-shared layers*/\\
Randomly assign value of $\lambda_1$, $\lambda_2$, and $\lambda_3$, such that $\lambda_1 +\lambda_2 + \lambda_3 =1$; set $\lambda_4\leftarrow1$ /*assumed in this work*/  \\
\smallskip
\SetKwRepeat{Do}{do}{while}
\Do{$(l \le  L{'})$}{
$l \leftarrow l+ \delta_{L}$\;
\smallskip
\nonl /*Obtaining dropout model from $M^{te}$ */\\
\smallskip
Call \textbf{Procedure 1}\;
\hspace{0.3cm}a.) Apply dropout on $l$ layers of $M^{te}$\; 
\hspace{0.3cm}b.) Obtain dropout model with $Q_b$ connections\;
\smallskip
Call \textbf{Procedure 2}\;
\hspace{0.3cm}a.) Solve optimization problem in Eq.~\ref{opti}\; 
\hspace{0.3cm}b.) Obtain compressed model $M^s$ from dropout model\;

\smallskip
\nonl /*Training of obtained model $M^s$*/\\
\smallskip
Obtain halting epoch $h$\;
\For{epoch $e\le E$}{
\If{$e\leq h$}{
Train $M^s$ using $M^{te}$ and $M^{tr}$\;
}

\Else{
\smallskip
\nonl /*Early halting technique*/\\
\smallskip
Train $M^s$ using $M^{tr}$\;
}
}
Solve optimization problem in Eq.~\ref{Main_optimization}\;
Obtain optimal value of $\lambda_1$, $\lambda_2$, and $\lambda_3$\; 
\smallskip
\nonl /*Append $\mathcal{L}_{comb}(\cdot)$ in list $\mathcal{P}[ \quad]$*/\\
\smallskip
$\mathcal{P}\leftarrow append(\mathcal{L}_{comb}(\cdot))$, preserve $M^{s}$  \;
$\delta_{L}\leftarrow L{'}/10$\;
}
\smallskip
$a\leftarrow \arg\min \{\mathcal{P}\}$\;
Obtain lightweight model $M^{s}$ for $\mathcal{L}_{comb}(\cdot)$ at $\mathcal{P}[a]$\;
$M^s$ is the appropriate student model for given edge device\;
\textbf{return} Optimal lightweight model $M^s$\; 
\end{algorithm}

\begin{figure*}[h]
\centering
\includegraphics[scale=0.95]{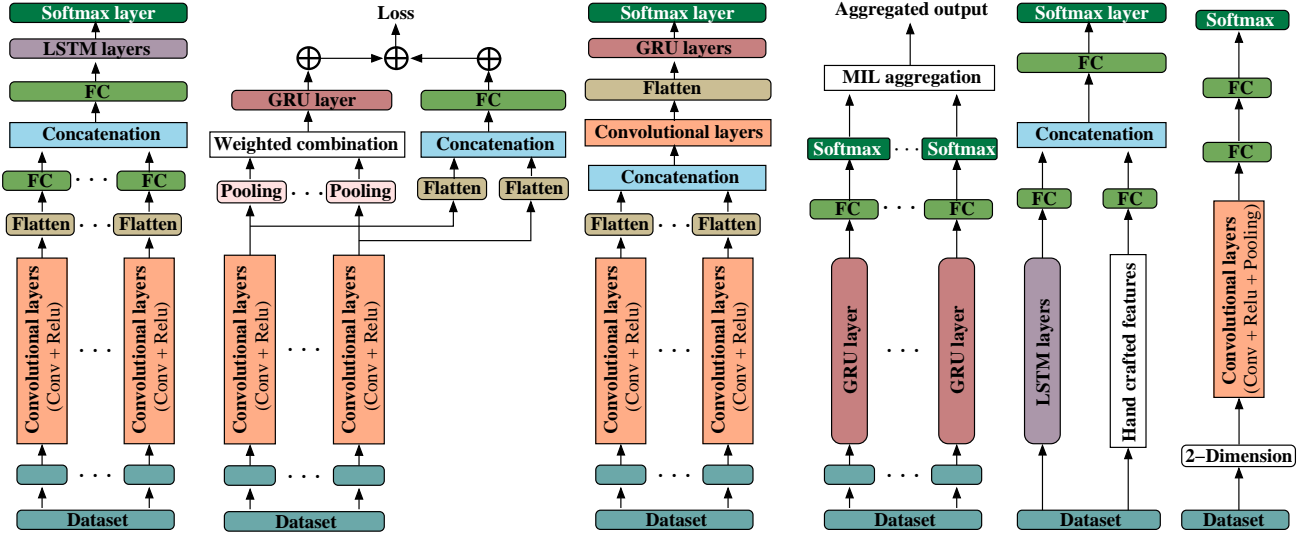}
\caption{Illustration of different DNN architectures incorporating sensory data for recognizing locomotion modes and human activities. (a) DeepZero~\cite{9164991}, (b) DeepFusion~\cite{xue2019deepfusion}, (c) DeepSense~\cite{yao2017deepsense}, (d) DT-MIL~\cite{janakiraman2018explaining}, (e) MFAD~\cite{chen2019smartphone}, and (f) HARM~\cite{noori2020human}.} 
\label{diff_models}
\vspace{-0.4cm}
\end{figure*}

\section{Empirical evaluation}
This section empirically evaluates the proposed work on publicly available datasets, existing large-size DNN, and edge devices. Our primary focus is to evaluate the effectiveness of the proposed work in transforming the large-size DNN to lightweight for edge devices.

\subsection{Evaluation setup}

\subsubsection{Large-size DNN $M^t$ architectures}\label{architecture}
We considered six existing DNN, including DeepZero~\cite{9164991}, DeepFusion~\cite{xue2019deepfusion}, DeepSense~\cite{yao2017deepsense}, DT-MIL~\cite{janakiraman2018explaining}, MFAP~\cite{chen2019smartphone}, and Human Activity Recognition using Multiple sensors fusion (HARM) [37], as shown in Fig.~\ref{diff_models}. These large-size DNN use sensory data for recognizing locomotion modes and human activities with high accuracy but require colossal parameters and FLOPs during their execution, as given in Table~\ref{result1}.

\begin{table}[h]
\caption{FLOPs and parameters required by large-size DNN, where $(A,B)$ is $A \times 10^B$.}
\resizebox{.49\textwidth}{!}{
\begin{tabular}{|c|p{0.6cm}|p{0.2cm}|p{0.6cm}|p{0.5cm}|c|c|}
\hline
\multirow{2}{*}{\textbf{DNN}} &\multicolumn{4}{l|}{\textbf{Number of DNN layers}}                            & \multirow{2}{*}{\textbf{\begin{tabular}[c]{@{}c@{}}FLOPs\end{tabular}}} & \multirow{2}{*}{\textbf{\begin{tabular}[c|]{@{}c@{}}Parameters\end{tabular}}} \\ \cline{2-5}
                                 & \textbf{Conv} & \textbf{FC} & \textbf{LSTM} & \textbf{GRU} &                                                                                          &                                                                                    \\ \hline
DeepZero~\cite{9164991}         &  $15$  &  $5$   &  $2$   &  \textemdash    & $(1.2,10)$  & $(3.8,7)$                                                                             \\ 
DeepFusion~\cite{xue2019deepfusion}& $18$   &  $3$   & \textemdash & $1$  & $(8.5,11)$                                                                                       & $(5.4,8)$                                                                                   \\ 
DeepSense~\cite{yao2017deepsense} & $12$   & \textemdash & \textemdash  &  $2$   & $(7.5,11)$                                                                                          & $(1.0,7)$                                                                                   \\ 
DT-MIL~\cite{janakiraman2018explaining} & \textemdash & $20$ & \textemdash  & $1$  & $(2.4,7)$                                                                                          & $(2.9,4)$                                                                                    \\ 
MFAP~\cite{chen2019smartphone} & \textemdash  & $2$  & $2$  & \textemdash & $(3.7, 7)$                                                                                          & $(1.2,7)$                                                                                   \\ 
HARM~\cite{noori2020human}   & $6$   & $2$  & \textemdash & \textemdash  & $(7.6,10)$                                                                                          & $(1.2,9)$                                                                                   \\ \hline
\end{tabular}
}
\label{result1}
\end{table}

\subsubsection{Datasets}
To evaluate the proposed work, we select four publicly available sensory datasets. These datasets are typically used in IoT applications, \textit{e.g.}, locomotion mode recognition (LMR)~\cite{shl2}, driving behaviour (DB)~\cite{vdb}, river pollution monitoring (RPM)~\cite{thoreau}. The specification of the datasets is mentioned in Table~\ref{datasets}. \#NC, \#SB, \#SR, \#TS, \#TR, and \#TE separately denote the number of classes, number of subjects, sampling rate (Hz), total samples, number of training samples, and number of testing samples, respectively.

\begin{table}[h]
\caption{Specifications of used datasets.}
\label{datasets}
 \resizebox{.485\textwidth}{!}{
\begin{tabular}{|@{}c|c|llllll@{}|}
\hline
\textbf{Dataset} & \multicolumn{1}{c|}{\textbf{Task}}&  \textbf{\#NC} & \textbf{\#SB}  & \textbf{\#SR}  & \textbf{\#TS} & \textbf{\#TR} & \textbf{\#TE}    \\ \hline
\textbf{SHL}~\cite{shl2} & LMR & $8$ & $4$ & $100$ & $22008$ & $16310$ & $5698$  \\ \hline
\textbf{VDB}~\cite{vdb} & DB & $5$  & $1$   & $10$  & $10000$  & $7000$  & $3000$  \\ \hline
\begin{tabular}[c]{@{}c@{}}\textbf{DBD}~\cite{dbd} \\ (upsampled)\end{tabular}& DB & $4$& $3$ & $2$& $11000$ & $7700$& $3300$  \\ \hline
\textbf{RWM}~\cite{thoreau} & RPM & $6$ & -    & $10$  & $100000$ & $70000$ & $30000$ \\ \hline
\end{tabular}
}
\end{table}

\subsubsection{Edge devices for running lightweight $M^s$} We consider five different edge devices for deploying the lightweight DNN, \textit{i.e.} trained student models ($M^s$), to verify the performance of the proposed approach. The devices include Intel edition kit ($\mathbf{d_1}$), Raspberry Pi $2$ ($\mathbf{d_2}$), Raspberry Pi $3$ ($\mathbf{d_3}$), Huwaie smartphone ($\mathbf{d_4}$), and Samsung smartphone ($\mathbf{d_5}$). The processing speed of the devices $\mathbf{d_1}$ to $\mathbf{d_{5}}$ are $11\times 10^8$, $3\times 10^9$, $5\times10^{10}$, $18\times 10^{10}$, and $29\times 10^{10}$ FLOPs/second, respectively.

\subsubsection{Baseline schemes for ablation studies}
Table~\ref{schemes1} summarizes the architecture of different lightweight $M^s$ and the training process for the ablation studies of the proposed work. Scheme $\mathbf{S}_5$ is the same as the proposed technique $\mathbf{S}_6$ with no early halting while training. 

\begin{table}[h]
\caption{Baseline schemes for ablation studies, where, $M^s$ (student) is lightweight DNN, and $M^{tr}$ (teacher) and $M^{te}$ (trainee) are large-size DNN.}
\label{schemes1}
\small
\begin{tabular}{|c|l|}
\hline
\textbf{Scheme} &\multicolumn{1}{c|}{\textbf{Description}}\\ \hline
\multirow{2}{*}{$\mathbf{S}_1$~\cite{hinton2015distilling}} & $M^s$ and $M^{tr}$ are independent\\ \cline{2-2}
                  & Training of $M^s$ guided by pre-trained $M^{tr}$
\\ \hline
\multirow{2}{*}{$\mathbf{S}_2$~\cite{mishra2017apprentice}} &  $M^s$ and $M^{te}$ are independent\\ \cline{2-2}
                   & $M^s$ and $M^{te}$ are trained simultaneously  
 \\ \hline
\multirow{2}{*}{$\mathbf{S}_3$~\cite{knowledge}}  &  $M^s$ is sub-model derive from $M^{te}$\\ \cline{2-2}
                   & $M^s$ and $M^{te}$ are trained simultaneously   
 \\ \hline
\multirow{2}{*}{$\mathbf{S}_4$~\cite{9151346}} &  $M^s$ is independent of $M^{tr}$ and $M^{te}$\\ \cline{2-2}
                   & $M^s$ and $M^{te}$ are trained under guidance of $M^{tr}$  
 \\ \hline
\multirow{2}{*}{$\mathbf{S}_5$} &  $M^s$ is sub-model derive from $M^{tr}$ or $M^{te}$\\ \cline{2-2}
                   & $M^s$ and $M^{te}$ are trained under guidance of $M^{tr}$  
 \\ \hline
\multirow{3}{*}{\begin{tabular}[c]{@{}c@{}}$\mathbf{S}_6$\\ (\textbf{proposed})\end{tabular}} &  $M^s$ is sub-model derive from $M^{tr}$ or $M^{te}$\\ \cline{2-2}
                    & \begin{tabular}[c]{@{}l@{}} $M^s$ and $M^{te}$ are trained under guidance by $M^{tr}$\\ up to some epochs then $M^{te}$ \textbf{training is halted} \end{tabular} 
\\ \hline
\end{tabular}
\end{table}

\subsubsection{Implementation details}
For implementing the lightweight DNN transformed from large-size DNN, as illustrated in Fig.~\ref{diff_models}, we incorporated the sequential model and functional API of deep learning library Keras in Python language. Next, Algorithm~$1$ and all the procedures are implemented in Python. We adopt the differential evolution technique for estimating the fractional contributions of the different loss functions, \textit{i.e.,} $\lambda_1$, $\lambda_2$, and $\lambda_3$. In the experimental analysis, we randomly divide the datasets into two sub-datasets, \textit{i.e.,} training and testing with $70\%$ and $30\%$ data instances, respectively, using the function $sklearn.model\_selection.train\_test\_split()$ in Sklearn model selection. We repeat each experiment $100$ times and calculate the average value. Further, $\alpha$ (memory size) is the maximum available memory on the edge while requesting the appropriate compressed DNN. $\beta$ (maximum allowable processing time) is estimated by analyzing the data processing history of the devices. In other words, $\beta$ is determined in accordance with the processing speed (FLOPs/seconds) achieved by the device in the past event of data processing. $\beta \ge$ prior device processing speed $\times$ current FLOPs. 
The lightweight DNN can be trained on a considered device or server. The training time substantially reduces in both cases while adopting the early halting scheme. This type of training reduces the training time and resources on devices. As resources are limited; thus, fast training could be beneficial. 

\vspace{-0.45cm}

\subsection{Validation metrics}\label{validation}
This work used the following standard classification metrics to evaluate and compare the performance of the FFL technique: F$_1$ score and accuracy. Let a given dataset consists of a set of $\mathcal{A}$ classes, and $\left | \mathcal{A} \right |$ represents the number of classes. Let $TP_i$, $TN_i$ $FP_i$, and $FN_i$ are the true positive, true negative, false positive, and false negative counts of a class  $i \in \mathcal{A}$, respectively. 
The \textit{accuracy} metric is computed as: 
\begin{equation}\small
\label{acc_m}
\frac{1}{\left | \mathcal{A} \right |}\sum_{i=1}^{\left | \mathcal{A} \right |}\frac{TP_i + TN_i}{TP_i+TN_i+FP_i+FN_i}.
\end{equation}
Next, the \textit{F$_1$ score} is computed as: 
\begin{equation}\small
\label{f1_m}
\frac{1}{\left | \mathcal{A} \right |}\sum_{i=1}^{\left | \mathcal{A} \right |} \frac{2 \times TP_i}{2\times TP_i+FP_i+FN_i}.
\end{equation}

\begin{table*}[h]
\caption{Illustration of accuracy ($\%$) achieved by schemes ($\mathbf{S}_1$-$\mathbf{S}_6$) on device specific student from teacher (DeepZero~\cite{9164991}, DeepFusion~\cite{xue2019deepfusion}, DeepSense~\cite{yao2017deepsense}, DT-MIL~\cite{janakiraman2018explaining}, MFAD~\cite{chen2019smartphone}, and HARM~\cite{noori2020human}). fConv and Conv are factorized and unfactorized Convolutional layers, fFC = factorized FC layer, cLSTM = coupled LSTM, MGU= Minimal Gated Unit.}
\centering
\label{result3}
\begin{tabular}{|p{0.5cm}|c|l|c|c|cccccc|}
\cline{2-11}
 \multicolumn{1}{c|}{} & \multirow{2}{*}{\textbf{\rotatebox{00}{Device}}} & \multicolumn{3}{c|}{\textbf{Student model specification}}                                 & \multicolumn{6}{c|}{\textbf{Accuracy (\%) in scheme}} \\ \cline{3-11} 
 \multicolumn{1}{c|}{}  &                                  & \multicolumn{1}{c|}{\textbf{Number of DNN layers}} & \textbf{FLOPs} & \textbf{Parameters} & $\mathbf{S}_1$~\cite{hinton2015distilling}     & $\mathbf{S}_2$~\cite{mishra2017apprentice}     & $\mathbf{S}_3$~\cite{knowledge}     & $\mathbf{S}_4$~\cite{9151346}    & $\mathbf{S}_5$    & \begin{tabular}[c]{@{}c@{}}$\mathbf{S}_6$\\ (\textbf{Proposed})\end{tabular} \\  \cline{1-11} 
\multicolumn{1}{|c|}{\multirow{8}{*}{\rotatebox{90}{\textbf{DeepZero~\cite{9164991}}}}} & $\mathbf{d_1}$  & fConv $=5$, fFC $=3$, cLSTM $=2$  & $1.0 \times 10^8$  & $1.1\times 10^6$ & $75.43$ & $79.27$ & $82.23$ & $87.27$ & $89.41$ &  $90.53$     \\ \cline{2-11} 
& $\mathbf{d_2}$  & fConv $=7$, fFC $=5$, cLSTM $=2$  & $2.1 \times 10^8$  & $1.6\times10^6$ 
& $79.42$ & $83.19$ & $85.19$ & $89.13$ & $91.21$ &  $91.08$     \\ \cline{2-11} 
& $\mathbf{d_3}$  & \begin{tabular}[l]{@{}l@{}}fConv $=12$, Conv $=3$, fFC $=5$,\\ LSTM $=1$, cLSTM$=1$ \end{tabular} & $2.2 \times 10^9$  & $1.9 \times 10^7$ 
& $82.31$ & $85.11$ & $87.93$ & $90.97$ & $92.54$  & $92.32$       \\ \cline{2-11} 
& $\mathbf{d_4}$  & \begin{tabular}[l]{@{}l@{}}fConv $=1$, Conv $=14$, fFC $=3$, \\FC $=2$, LSTM $=1$, cLSTM $=1$ \end{tabular} & $4.5 \times 10^9$  & $3.1 \times 10^7$                  & $85.41$ & $87.29$ & $88.71$ & $92.23$ & $93.51$  &  $93.33$      \\ \cline{2-11} 
& $\mathbf{d_5}$  & \begin{tabular}[l]{@{}l@{}}fConv $=1$, Conv $=14$, fFC $=1$,\\ FC $=4$, LSTM $=1$, cLSTM $=1$ \end{tabular} & $4.8\times 10^9$  & $3.2\times 10^7$  & $86.21$ & $87.45$ & $88.92$ & $92.71$ & $93.44$ &  $93.57$ \\ \hline \hline

\multicolumn{1}{|c|}{\multirow{7}{*}{\rotatebox{90}{\textbf{DeepFusion~\cite{xue2019deepfusion}}}}}
& $\mathbf{d_1}$  & fConv $=6$, fFC $=2$, MGU $=1$  & $1.3 \times 10^{10}$  & $1.0\times10^8$  & $85.93$ & $88.81$ & $89.78$ & $90.07$  & $91.13$ & $91.19$       \\ \cline{2-11} 
& $\mathbf{d_2}$  & fConv $=6$, fFC $=3$, GRU $=1$  & $1.6 \times 10^{10}$  & $1.1\times10^8$  & $86.17$ & $89.43$ & $90.03$ & $90.56$  & $92.83$ & $92.16$       \\ \cline{2-11} 
& $\mathbf{d_3}$  & fConv $=7$, fFC $=3$, GRU $=1$ & $1.4 \times 10^{11}$  & $2.0 \times 10^8$  & $88.29$ & $90.97$ & $91.07$ & $92.43$  & $94.07$ & $93.23$      \\ \cline{2-11} 
& $\mathbf{d_4}$  & \begin{tabular}[l]{@{}l@{}} fConv $=12$, Conv $=3$, fFC $=3$,\\  MGU $=1$ \end{tabular} & $3.0 \times 10^{11}$  & $2.8 \times 10^8$ & $89.31$ & $91.34$ & $92.23$ & $92.71$ & $94.93$   & $94.57$       \\ \cline{2-11} 
& $\mathbf{d_5}$  &  \begin{tabular}[l]{@{}l@{}}  fConv $=10$, Conv $=5$, fFC $=3$,\\ MGU $=1$ \end{tabular} & $3.4 \times 10^{11}$  & $3.1 \times 10^8$  & $89.73$ & $91.47$ & $92.31$ & $92.97$ & $95.03$   & $94.63$  \\ \hline \hline

\multicolumn{1}{|c|}{\multirow{7}{*}{\rotatebox{90}{\textbf{DeepSense~\cite{yao2017deepsense}}}}}
& $\mathbf{d_1}$  & fConv $=8$, GRU $=1$  & $6.1 \times 10^{10}$  & $1.0\times10^6$ & $80.22$ & $80.83$ & $84.21$ & $84.89$  & $89.12$ & $89.47$       \\ \cline{2-11} 
& $\mathbf{d_2}$  &  \begin{tabular}[l]{@{}l@{}} fConv $=6$, Conv $=4$, GRU $=1$,\\ MGU $=1$ \end{tabular} & $9 \times 10^{10}$  & $1.1\times10^6$ & $81.29$ & $81.92$ & $87.20$ & $87.62$  & $92.53$ & $92.21$ \\ \cline{2-11} 
& $\mathbf{d_3}$  & fConv $=6$, Conv $=4$,GRU $=2$ & $1.1 \times 10^{11}$  & $1.7 \times 10^6$ & $82.06$ & $82.81$ & $88.17$ & $88.97$  & $92.91$ & $93.03$      \\ \cline{2-11} 
& $\mathbf{d_4}$  & \begin{tabular}[l]{@{}l@{}}fConv $=4$, Conv $=8$, GRU $=1$,\\ MGU $=1$ \end{tabular} & $2.1 \times 10^{11}$  & $2.3 \times 10^6$ & $82.72$ & $83.09$ & $88.51$ & $89.23$ & $93.09$   & $93.20$       \\ \cline{2-11} 
& $\mathbf{d_5}$  & fConv $=4$, Conv $=8$,GRU $=2$ & $4.3 \times 10^{11}$  & $5.2 \times 10^6$  & $82.91$ & $83.93$ & $89.02$ & $89.91$ & $93.37$   & $93.44$       \\  \hline \hline

\multicolumn{1}{|c|}{\multirow{5}{*}{\rotatebox{90}{\textbf{DT-MIL~\cite{janakiraman2018explaining}}}}}
& $\mathbf{d_1}$  & fFC $=12$, MGU $=1$  & $3.2 \times 10^{6}$  & $8.1\times10^3$  & $73.93$ & $78.61$ & $79.53$ & $84.21$  & $87.23$ & $87.07$       \\ \cline{2-11} 
& $\mathbf{d_2}$  & fFC $=12$, FC $=2$, MGU $=1$  & $4 \times 10^{6}$  & $9\times10^3$  & $74.27$ & $78.83$ & $80.27$ & $84.61$  & $88.73$ & $88.26$       \\ \cline{2-11} 
& $\mathbf{d_3}$  & fFC $=12$, FC $=6$, MGU $=1$ & $6 \times 10^{6}$  & $1.1 \times 10^4$ & $75.29$ & $81.44$ & $81.93$ & $85.07$  & $89.45$ & $89.03$      \\ \cline{2-11} 
& $\mathbf{d_4}$  & fFC $=12$, FC $=8$, MGU $=1$ & $9 \times 10^{6}$  & $1.2 \times 10^4$  & $75.89$ & $81.87$ & $82.29$ & $85.27$ & $90.83$   & $90.85$       \\ \cline{2-11} 
& $\mathbf{d_5}$  & fFC $=8$, FC $=12$, GRU $=1$ & $1.1 \times 10^{7}$  & $1.6 \times 10^4$  & $76.39$ & $82.91$ & $83.17$ & $85.59$ & $91.21$   & $91.07$  \\  \hline \hline

\multicolumn{1}{|c|}{\multirow{7}{*}{\rotatebox{90}{\textbf{MFAD~\cite{chen2019smartphone}}}}}
& $\mathbf{d_1}$ & fFC $=1$, FC $=1$, LSTM $=1$  & $2.1\times 10^{7}$ & $7.6 \times 10^6$ & $79.17$ & $82.08$ & $83.07$ & $84.21$ & $85.37$   & $85.16$       \\ \cline{2-11} 
& $\mathbf{d_2}$  & fFC $=1$, FC $=1$, LSTM $=1$  & $2.1\times 10^{7}$ & $7.6 \times 10^6$ & $79.29$ & $82.31$ & $83.26$ & $84.53$ & $85.81$   & $85.23$       \\ \cline{2-11} 
& $\mathbf{d_3}$  & \begin{tabular}[l]{@{}l@{}} fFC $=1$, FC $=1$, LSTM $=1$,\\ cLSTM $=1$ \end{tabular}  & $2.4\times 10^{7}$ & $9.1 \times 10^6$ & $80.17$ & $83.23$ & $84.11$ & $85.38$ & $87.47$   & $87.13$       \\ \cline{2-11} 
& $\mathbf{d_4}$  & \begin{tabular}[l]{@{}l@{}} fFC $=1$, FC $=1$, LSTM $=1$,\\ cLSTM $=1$ \end{tabular} & $2.4\times 10^{7}$ & $9.1 \times 10^6$ 
& $80.41$ & $83.63$ & $84.61$ & $85.83$ & $87.65$   & $87.29$       \\ \cline{2-11} 
& $\mathbf{d_5}$  & fFC $=1$, FC $=1$, LSTM $=2$  & $3.2\times 10^{7}$ & $1.0 \times 10^7$ 
& $81.01$ & $84.33$ & $85.07$ & $86.51$ & $88.27$   & $88.02$     \\  \hline \hline

\multicolumn{1}{|c|}{\multirow{5}{*}{\rotatebox{90}{\textbf{HARM~\cite{noori2020human}}}}}
& $\mathbf{d_1}$  & fConv $=1$, fFC $=2$ & $4.5 \times 10^{9}$  & $6.1\times10^7$  & $76.71$  & $80.49$  & $82.96$ & $87.81$ & $90.09$ & $90.13$ \\ \cline{2-11}
& $\mathbf{d_2}$  & fConv $=1$, fFC $=2$ & $4.5 \times 10^{9}$  & $6.1\times10^7$  & $77.39$  & $81.37$  & $84.76$ & $89.03$ & $90.39$ & $90.23$ \\ \cline{2-11} 
& $\mathbf{d_3}$  & fConv $=2$, fFC $=2$ & $1.9 \times 10^{10}$  & $3.1 \times 10^8$ & $80.01$  & $83.23$ & $85.59$ & $89.47$  & $91.92$  & $90.93$      \\ \cline{2-11} 
& $\mathbf{d_4}$  & fConv $=4$ fFC $=2$ & $2.9 \times 10^{10}$  & $4.4 \times 10^8$  & $80.22$ & $83.63$ & $85.81$ & $89.93$ & $92.17$ & $91.71$       \\ \cline{2-11} 
& $\mathbf{d_5}$  & fConv $=4$, Conv $=1$ fFC $=2$ & $3.4 \times 10^{10}$  & $5.1 \times 10^8$ & $80.71$ & $83.61$ & $86.21$ & $90.21$ & $92.19$ & $92.34$ \\ \hline
\end{tabular}
\end{table*}

\subsection{Experimental results}
This section carries out the experimental evaluations to illustrate the impact of the following on performance of proposed approach: schemes $\mathbf{S}_1$-$\mathbf{S}_6$, loss functions, duration of training, size of lightweight  $M^s$, early halting of training of $M^{te}$, and datasets  in Section~\ref{resulte11}, Section~\ref{resulte21}, Section~\ref{resulte31}, Section~\ref{resulte41}, Section~\ref{resulte51}, and Section~\ref{resulte61}, respectively.

\subsubsection{Impact of different schemes on accuracy of DNN} \label{resulte11}
First, we performed the experiment to estimate the performance of lightweight DNN train using different schemes $\mathbf{S}_1$-$\mathbf{S}_6$. We considered different large-size DNN and $\mathbf{d_1}$ to $\mathbf{d_{5}}$ edge devices. The value of constraints $\alpha=0.65$ and $\beta=180$ ms. Table~\ref{result3} illustrates the configuration of lightweight DNN based on the available resources on edge. The result depict that the accuracy of the lightweight DNN transform using scheme $\mathbf{S}_6$ is almost equal to the scheme $\mathbf{S}_5$. However, $\mathbf{S}_5$ requires a large number of FLOPs and parameters during the training of lightweight DNN. Due to the early halting of the trainee model training, the scheme $\mathbf{S}_6$ achieves a significant reduction in training time of lightweight DNN $M^s$. It is interesting to observe that the transformed lightweight DNN from large-size DNN using scheme $\mathbf{S}_6$ achieves high accuracy within the constraints of edge device, take less time during training of lightweight model and therefore saves the energy and resources of training machine. 

\subsubsection{Impact of the loss functions on accuracy of DNN} \label{resulte21}
Table~\ref{lambda} illustrates the fractional contribution of different loss functions on the performance of the lightweight DNN $M^s$ using different edge devices. We used the devices and available resources as shown in previous results. We considered DeepZero as a large-size DNN (teacher model). We can observe from the result that with the increase in the device's resources, the accuracy and F$_1$ score of the lightweight DNN $M^s$ improved, and the contribution of distillation loss  ($\lambda_3$) increases. It is because when the difference between $M^s$ and $M^{te}$ is significant then simultaneous training deviates the $M^s$ from achieving optimal convergence point due to random initialization of $M^{te}$.

\begin{table}[h]
\centering
\caption{Fractional contributions ($\lambda_1$, $\lambda_2$, and $\lambda_3$) of different loss functions on the performance of $M^{s}$ with 
$\mathbf{d_1}$ to $\mathbf{d_5}$ devices. }
\resizebox{.49\textwidth}{!}{
\begin{tabular}{cccccc}
\hline
\multirow{2}{*}{\textbf{Device}} & \multicolumn{3}{c}{\textbf{Fractional weights}} & \multirow{2}{*}{\textbf{\begin{tabular}[c]{@{}c@{}}Accuracy\end{tabular}}} & \multirow{2}{*}{\textbf{\begin{tabular}[c]{@{}c@{}}F$_1$ score\end{tabular}}} \\ \cline{2-4}
                & $\lambda_1$ &  $\lambda_2$ & $\lambda_3$  &   &  \\ \hline
$\mathbf{d_1}$  & $0.5117$ & $0.3972$ & $0.0911$  &  $90.53\%$  & $91.21\%$ \\ \hline
$\mathbf{d_2}$  & $0.4919$ & $0.4523$ & $0.0563$  &  $91.08\%$  & $92.74\%$ \\ \hline
$\mathbf{d_3}$  & $0.3208$ & $0.3563$ & $0.3227$  &  $92.32\%$  & $93.98\%$ \\ \hline
$\mathbf{d_4}$  & $0.3700$ & $0.2965$ & $0.3334$  &  $93.33\%$  & $94.90\%$ \\ \hline
$\mathbf{d_5}$  & $0.3922$ & $0.2717$ & $0.3361$  &  $93.57\%$  & $95.13\%$ \\ \hline
\end{tabular}
}
\label{lambda}
\end{table}

\subsubsection{Impact of the training time on accuracy}\label{resulte31}
In this experiment, we determine the training time and accuracy achieved under different schemes for $M^{s}$. We consider DeepZero as large-size DNN  $M^{tr}$, whose lightweight variant is deployed on device $\mathbf{d_3}$. Table~\ref{time} illustrates the training time and accuracy of different schemes for device $\mathbf{d_3}$ excluding the training time of the pre-trained teacher model. As shown in the previous result, $\mathbf{S}_5$ and $\mathbf{S}_6$ give the high accuracy as compared with others. $\mathbf{S}_5$ trains both $M^s$ and $M^{te}$ simultaneously and therefore needs more FLOPs. The proposed $\mathbf{S}_6$ early halts the training of $M^{te}$ and needs fewer resources. Therefore, an edge device with limited resources takes more time to train $\mathbf{S}_5$ as compared to the proposed $\mathbf{S}_6$. Table~\ref{time} illustrates $\mathbf{S}_6$ has the best accuracy within the given training time. This is because other schemes either lack sufficient training or design lightweight DNN randomly, resulting in lower accuracy.  

\begin{table}[h]
\caption{Training time and accuracy on different schemes for device $\mathbf{d_3}$ on large-size DNN (DeepZero).}
 \resizebox{.49\textwidth}{!}{
\begin{tabular}{ccccccc}
\hline
\textbf{Schemes}       & $\mathbf{S}_1$ & $\mathbf{S}_2$ & $\mathbf{S}_3$ & $\mathbf{S}_4$ & $\mathbf{S}_5$ & $\mathbf{S}_6$\\ \hline
\multicolumn{7}{c} {\textbf{Part (a): Accuracy v/s required training time (in minutes)}} \\
\textbf{Accuracy} (in $\%$) & $82.31$ & $85.11$ & $87.93$ & $90.97$ & $92.54$  & $92.32$                                                                  \\ 
\begin{tabular}[c]{@{}c@{}}\textbf{Training time} $\pm 3$ \end{tabular} & $83$ & $219$ & $193$ & $231$ & $207$ & $163$                                                                 \\ \hline
\multicolumn{7}{c} {\textbf{Part (b): Accuracy on a given training time $\mathbf{=180}$ minutes}} \\
\begin{tabular}[c]{@{}c@{}}\textbf{FLOPs ($\times 10^{13}$)}\end{tabular} & $1.09$ & $2.37$ & $2.37$ & $2.37$ & $2.37$ & $2.15$                                                               \\ 
\textbf{Accuracy} (in $\%$)  & $82.03$ & $69.95$ & $82.03$ & $70.88$ & $80.47$  & $92.32$                                                                  \\ \hline
\end{tabular}
}
\label{time}
\end{table}

\subsubsection{Compression ratio of large-size DNN}\label{resulte41}
In this section, we illustrate the impact of the compression ratio (size of $M^s$/size of $M^t$) on the accuracy of the lightweight DNN. The compression ratio depends on the available resources of the edge device. An edge device requires a high compression ratio with low processing speed (FLOPs) on a fixed MAP time $\beta$. Table~\ref{compression} illustrates the various compression ratios of DeepZero. As expected, the high compression ratio gives low accuracy and F$_1$ score. An interesting observation from this result is that the accuracy and F$_1$ score go down sharply after a fixed compression ratio. The lightweight DNN for the given compression has very few layers and gated units. Thus, it shows incompetence in successfully classifying the given classes. Additionally, F$_1$ score is higher than the achieved accuracy due to the uneven distribution of class labels. 

\begin{table}[h]
\caption{Different compression ratio of large-size DNN (DeepZero).}
\resizebox{.49\textwidth}{!}{
\begin{tabular}{c|cccccc}
\hline
\textbf{\begin{tabular}[c]{@{}c@{}}Compression\\ ratio\end{tabular}}  & $\times60$ & $\times50$ & $\times43$  & $\times18$& $\times13$ & $\times4.8$  \\ \hline 
\textbf{\begin{tabular}[c]{@{}c@{}}Accuracy (in \%)\end{tabular}} & $67.43$ & $70.17$ & $73.22$ & $87.21$ & $90.23$ & $92.16$ \\ \hline
\textbf{\begin{tabular}[c]{@{}c@{}}F$_1$ score (in \%)\end{tabular}} & $69.09$ & $71.83$ & $74.88$ & $88.87$ & $91.89$ & $93.82$ \\ \hline
\textbf{\begin{tabular}[c]{@{}c@{}}FLOPs ($\times 10^{9}$)\end{tabular}} & $2.03$  & $2.40$  & $2.79$ & $6.67$ & $9.23$ & $25.21$ \\ \hline
\end{tabular}
}
\label{compression}
\end{table}

\subsubsection{Impact of halting of $M^{te}$ on performance of $M^s$}\label{resulte51}
Further, we depict the impact of the halting time of $M^{te}$ on the performance of lightweight $M^s$. We used the SHL dataset and the DeepZero model. Table~\ref{halting} illustrates that after a fixed duration of training of $M^{te}$, the progress in the accuracy and F$_1$ score of $M^s$ is almost constant. However, the required resources for continuous training of $M^{te}$ is increased with time. We, therefore, conclude that training of $M^{te}$ and $M^s$ till the end of processing does not provide high accuracy. They only consume more resources. Fig.~\ref{halting} illustrates the saturation in accuracy and F$_1$ score after a certain epochs ($60$ for DeepZero). After this saturation point, we can quickly halt the training of $M^{te}$ without compromising accuracy of $M^s$.   

\begin{table}[h]
\caption{Impact of halting training of trainee model on accuracy and F$_1$ score achieved by $M^s$ of DeepZero~\cite{9164991} for device $\mathbf{d_2}$.}
\centering
 \resizebox{.49\textwidth}{!}{
\begin{tabular}{l@{}ccccccc}
\hline
\textbf{\begin{tabular}[c]{@{}c@{}}\textbf{Epochs}\end{tabular}}  & $40$   & $50$    & $\mathbf{60}$ & $70$ & $80$ & $90$ & $100$   \\ \hline
\begin{tabular}[c]{@{}c@{}}\textbf{Training time}\\ $M^{te}$ (in min.)\end{tabular} & $141$ & $154$ & $\mathbf{163}$ & $172$ & $183$ & $198$   & $207$   \\ \hline
\begin{tabular}[c]{@{}c@{}}Accuracy \\ $M^s$ (in $\%$)\end{tabular}    & $78.51$ & $83.02$ & $\mathbf{91.08}$ & $91.37$ & $91.53$ & $91.97$ & $92.54$ \\ \hline
\begin{tabular}[c]{@{}c@{}}F$_1$ score \\ $M^s$ (in $\%$)\end{tabular}    & $80.17$ & $84.68$ & $\mathbf{92.74}$ & $93.03$ & $93.19$ & $93.63$ & $94.20$ \\ \hline
\begin{tabular}[c]{@{}c@{}}FLOPs\\ ($\times 10^{13}$) \end{tabular}  & $2.53$ & $2.77$ & $\mathbf{2.93}$ & $3.09$ & $3.29$ & $3.56$ & $3.73$ \\ \hline
\end{tabular}
}
\label{halting}
\end{table}

\subsubsection{Impact of datasets on the accuracy of DNN}\label{resulte61}
Finally, we study the performance achieved by different schemes ($\mathbf{S}_1$-$\mathbf{S}_6$) on selected datasets (SHL, VDB, DBD, and RWM). Fig.~\ref{dataset}(a) illustrates the average accuracy (in \%) achieved on schemes $\mathbf{S}_1$-$\mathbf{S}_6$. The result illustrates that for the DBD dataset, the accuracy under each scheme is highest. It is due to the least number of classes in the DBD dataset. Similarly, SHL achieves the lowest accuracy due to the presence of a maximum $8$ classes. Next, $\mathbf{S}_5$ and $\mathbf{S}_6$ achieve the highest accuracy on all the datasets (\textit{i.e.,} SHL, VDB, DBD, and RWM). Scheme $\mathbf{S}_5$ slightly supersedes the accuracy of scheme $\mathbf{S}_6$, as it incorporates training of trainee for all epochs. However, $\mathbf{S}_5$ consumes higher resources than $\mathbf{S}_6$. Further, a similar variation in F$_1$ score is observed under schemes $\mathbf{S}_1$-$\mathbf{S}_6$ on different datasets, as illustrated in Fig.~\ref{dataset}(b). F$_1$ score is higher than the accuracy for all datasets. The class labels are not uniformly distributed among all class labels in the considered datasets. An interesting observation from the result is that if the number of class labels in the dataset is small then the achieved accuracy will be higher. Additionally, if the distribution of class labels is non-uniform then F$_1$ score will be higher in contrast with accuracy. 

\begin{figure}[h]
        \centering
\begin{tabular}{cc}
\hspace{-0.2cm}\includegraphics[scale=0.55]{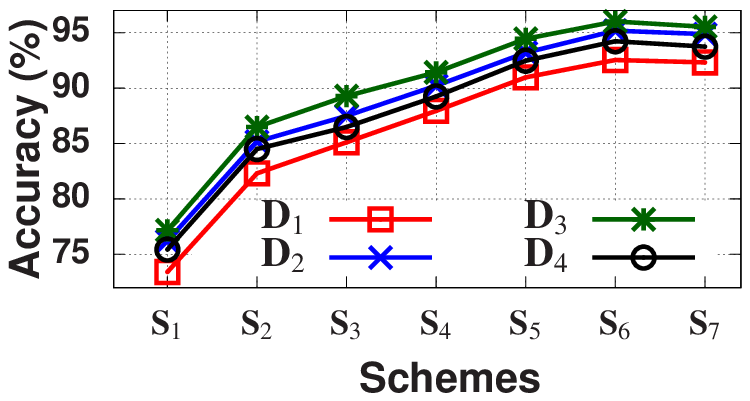}  & \hspace{-0.3cm} \includegraphics[scale=0.55]{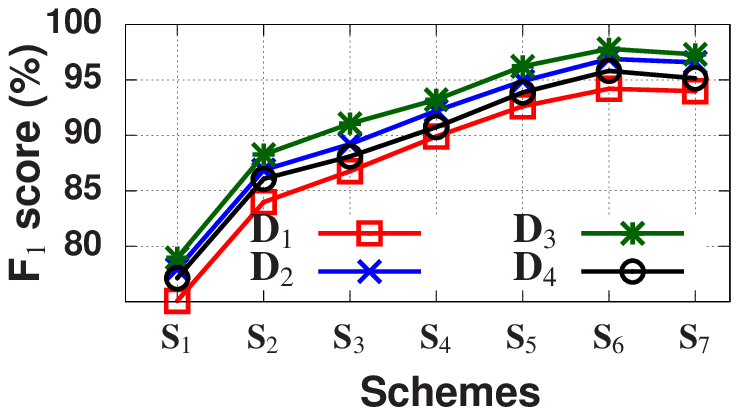}\vspace{-.1cm} \vspace{-.1cm}\\
\hspace{0.3cm}\scriptsize{(a) Accuracy.} & \hspace{0.3cm} \scriptsize{(b) F$_1$ score.}
\end{tabular}
        \caption{Impact of datasets (SHL, VDB, DBD, and RWM) on accuracy and F$_1$ score of different schemes.} 
        \label{dataset}
        \vspace{-0.6cm}
\end{figure}

\section{Real-world evaluation}\label{evaluate_model}
This section presents a real-world application of the proposed lightweight DNN to recognize unseen locomotion modes on edge devices. 
To verify the effectiveness of the proposed approach, we collected locomotion mode recognition dataset for evaluation. 
\vspace{-0.5cm}
\subsection{Hardware and software}
The prototype hardware is based on the NodeMCU ESP32 as data collection and processing unit. It is powered by ESP8266 module that can wirelessly transmit data using WiFi. Next, we attach the inertial sensors to measure angular rate, force and magnetic field and transfer to the NodeMCU. Further, the data is processed on NodeMCU using deployed lightweight DNN to predict class labels (locomotion modes). These labels are transferred to the server using the GSM module. NodeMCU has $512$ KB SRAM (with $4$ MB flash storage); therefore a high order DNN compressed is needed. We considered two scenarios of locomotion mode recognition, \textit{i.e.,} identifying locomotion modes using sensors deployed in the shoes of the kids and the wrist band of a person. 
 
We used DeepZero~\cite{9164991} as the large-size DNN upon which transform is performed. Fig.~\ref{diff_models}(a) illustrates the architecture of the DeepZero. The lightweight DNN of DeepZero is exported and loaded into the flash memory of the NodeMCU. The large-size DeepZero is trained on the Dell PC with $32$ GB RAM with a clock speed of $2.4$ GHz. The pre-trained model is further transformed using the proposed scheme and deployed on NodeMCU. Besides, the compressed DNN (from DeepZero) is trained on the server, as the available memory on NodeMCU is limited to store training data on its primary storage. 

\vspace{-0.5cm}
\subsection{Data collection} 
We collected the sensory data of different locomotion modes including bicycle ($\mathbf{a_1}$), bike ($\mathbf{a_2}$), car ($\mathbf{a_3}$), auto rickshaw ($\mathbf{a_4}$), bus ($\mathbf{a_5}$), and train ($\mathbf{a_6}$). To facilitate the data collection, we developed an android application that uses the Inertial Measurement Unit (IMU) sensor of the smartphone. The sampling rate of IMU is set to $100$ Hz to record $6000$ data points per minute. We used an android smartphone, Samsung Galaxy Alpha for collecting data against each locomotion modes. The data was collected by the $10$ volunteers ($5$ males and $5$ females). The android application consists of a menu through which volunteers can select a locomotion mode. The measurements of the IMU sensor is recorded for $60$ seconds. 

\subsection{Evaluation methods}
We considered the following four evaluation methods for verifying the effectiveness of the proposed approach in the real world scenario. First, we used \textbf{Baseline} method that is a lightweight version of DeepZero~\cite{9164991} and NodeMCU ESP32 as edge device. Next, we used \textbf{KD$_1$} an extension of the baseline method, where the lightweight DNN is trained under the guidance of the pre-trained DeepZero model using the knowledge distillation technique discussed in~\cite{hinton2015distilling}. Next, \textbf{KD$_2$} method where some initial layers of lightweight DNN and standard DeepZero are shared to improve the performance of lightweight DNN. The lightweight and standard models are trained simultaneously using the knowledge distillation technique discussed in~\cite{knowledge}. Finally, we used the proposed approach, named as \textbf{Proposed} in the results.  Apart from the existing methods, we adopt an early halting technique for training the lightweight DNN under the guidance of pre-trained and untrained DeepZero. The initial layers of the compressed DNN are shared with the large-size DeepZero. 

\subsection{Validation metrics} 
\noindent $\bullet$ \textit{F$_1$ score and accuracy:} The description of the validation metrics F$_1$ score and accuracy are discussed in Section~\ref{validation}. The equations of F$_1$ score and accuracy are given in Eq.~\ref{acc_m} and Eq.~\ref{f1_m}, respectively.\\
 \noindent $\bullet$ \textit{Precision:} Precision of a DNN is defined as the ratio of correct positive observation to the total correctly predicted observation, \textit{i.e.}, $\mathbf{P}_3=\frac{1}{\left | \mathcal{A} \right |}\sum_{i=1}^{\left | \mathcal{A} \right |}\frac{TP_i}{TP_i+FP_i}$.\\
\noindent $\bullet$ \textit{Leave-one-out test:} This validation metric trains the DNN for all class labels except for one randomly chosen class label. However, during testing, the unseen class label is also supplied for predicting the output. Thus, it evaluates the performance of the classifier for unseen class labels.\\
\vspace{-0.5cm}
\subsection{Result 1: Impact of memory and execution time}\label{rw1}
We first study the impact of memory and execution time on the validation metrics. Fig.~\ref{memoryrw}(a1) illustrates the accuracy achieved by the different methods with the change in memory ratio. The memory ratio is the ratio of required memory to the available memory of edge device. We can observe from the results that the proposed work outperforms the existing methods and achieves significantly higher accuracy with minimal energy consumption. The proposed work achieves accuracy around $94\%$ when memory ratio is just $0.65$. Similar observations can be made for other validation metrics, \textit{i.e.,} F$_1$ score, and precision, as shown in parts (a2)-(a3) of Fig.~\ref{memoryrw}, respectively. Next, parts (b1)-(b3) of Fig.~\ref{memoryrw} illustrate the impact of execution time on the validation metrics. The results depict that the execution time also follows a similar pattern as memory consumption, where the proposed work outperforms the existing methods and achieve maximal accuracy in minimum execution time. It requires $180$ ms to achieve the performance of more than $93$\%. It is because of the involvement of multiple teachers (teacher and trainee) and layer sharing in the proposed scheme. Last, we study the impact of simultaneous change in memory ratio and execution time on the validation metrics as shown in parts (c1)-(c3) of Fig.~\ref{memoryrw}. Similar as previous results, the results demonstrate that the proposed work can achieve accuracy of around $93\%$ when the execution time ($\beta$) is just $180$ ms, and memory ratio ($\alpha$) is $0.65$. 

\vspace{-.2cm}
\begin{figure}[h]
        \centering
\begin{tabular}{ccc}
\hspace{-.3cm}\includegraphics[scale=0.36]{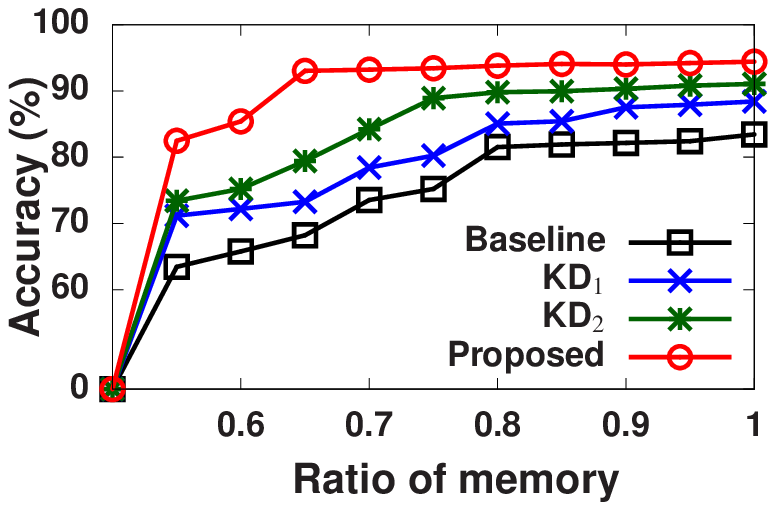}  & \hspace{-.3cm}\includegraphics[scale=0.36]{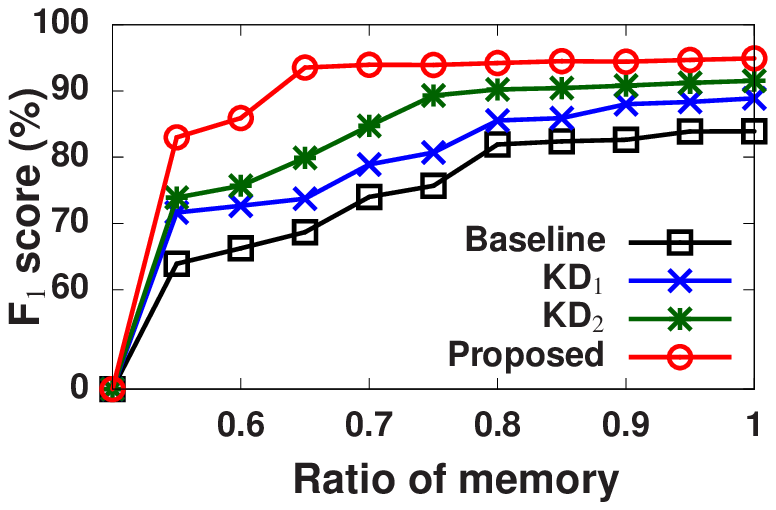}&\hspace{-.3cm}\includegraphics[scale=0.36]{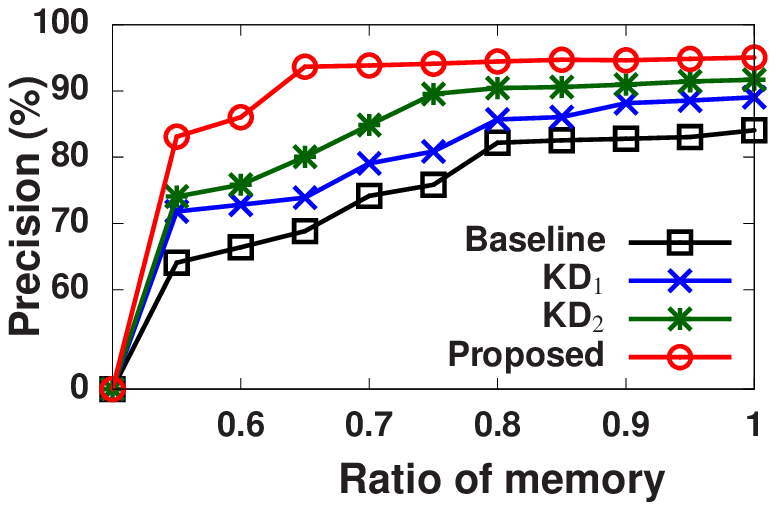}\vspace{-.1cm}\\
\scriptsize{(a1) Accuracy.} & \scriptsize{(a2) F$_1$ score.} & \scriptsize{(a3) Precision.} \\
\hspace{-.3cm}\includegraphics[scale=0.36]{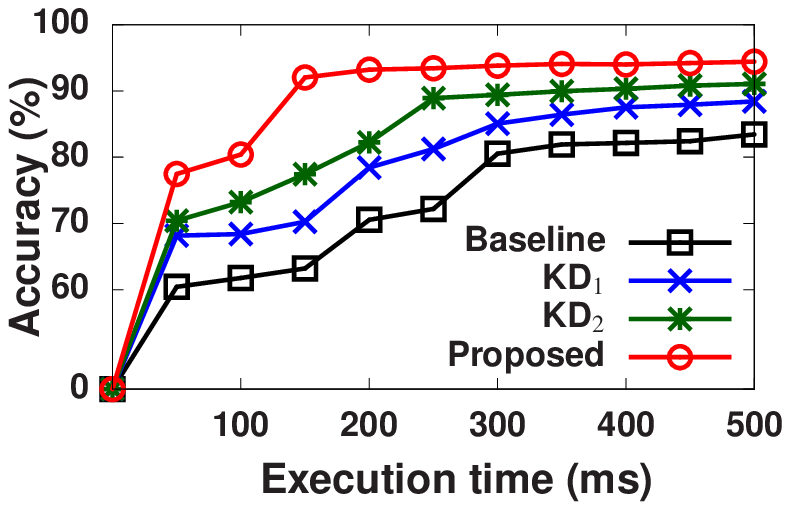}  & \hspace{-.3cm}\includegraphics[scale=0.36]{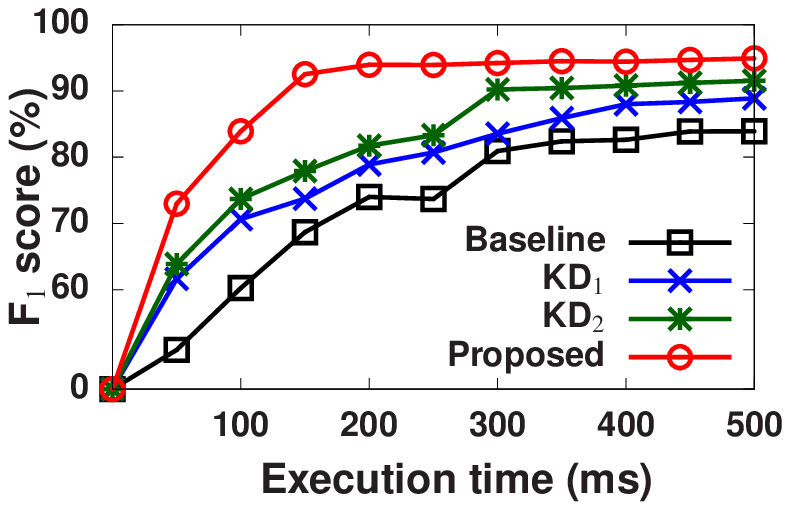}&\hspace{-.3cm}\includegraphics[scale=0.36]{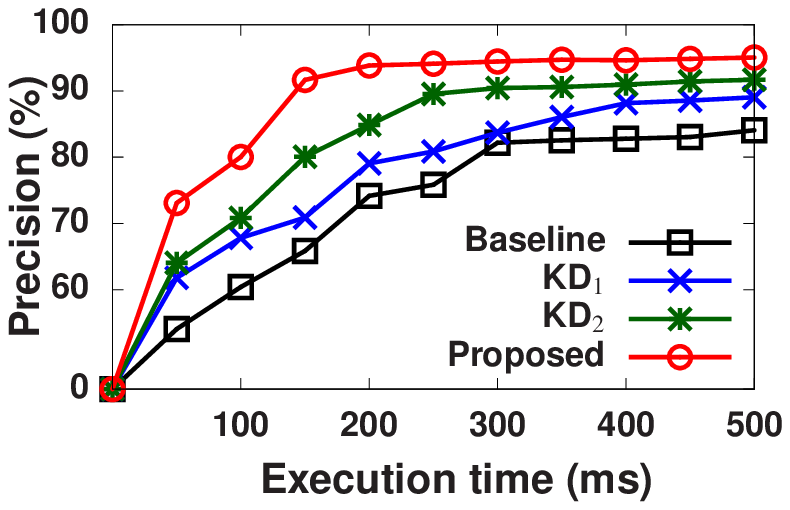}\vspace{-.1cm}\\
\scriptsize{(b1) Accuracy.} & \scriptsize{(b2) F$_1$ score.} & \scriptsize{(b2) Precision.}\\
\hspace{-.3cm}\includegraphics[scale=0.34]{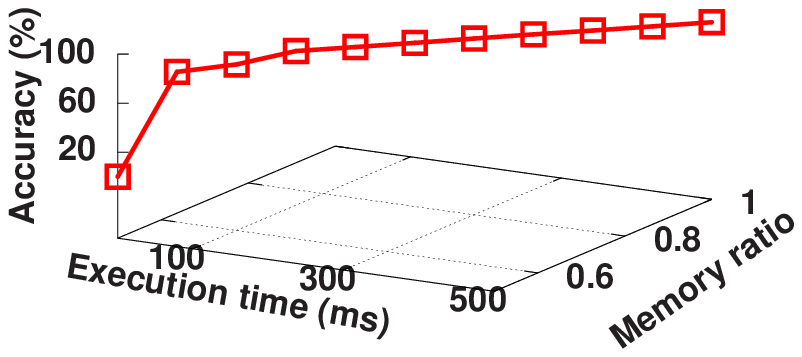}  & \hspace{-.3cm}\includegraphics[scale=0.34]{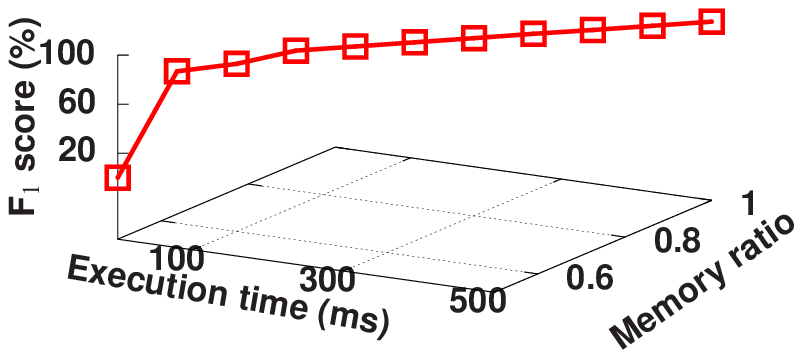}&\hspace{-.3cm}\includegraphics[scale=0.34]{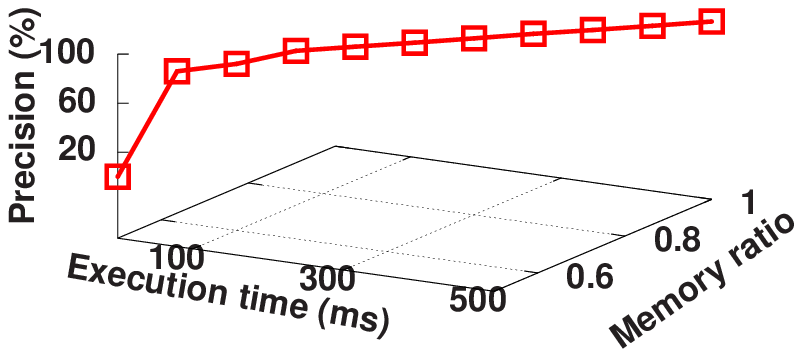} \vspace{-.1cm}\\
\scriptsize{(c1) Accuracy.} & \scriptsize{(c2) F$_1$ score.} & \scriptsize{(c3) Precision.} 
\end{tabular}
        \caption{Impact of memory, execution time, and simultaneous change in memory ratio and execution time on validation metrics.} 
        \label{memoryrw}
        \vspace{-.4cm}
\end{figure}

\subsection{Result 2: Class-wise accuracy} 
In this result, we fixed the value of edge constraints $\alpha=0.65$ and $\beta=180$ ms for estimating the class-wise accuracy of different methods that were considered for real world evaluation. Table~\ref{rwt1} illustrates the class-wise accuracy of different methods in the real world evaluation. We can observe from the result that the proposed method outperforms all existing methods in achieving accuracy against each class. Additionally, the class-wise accuracy of class $\mathbf{a_2}$ (bike) is highest, as it holds the most identifiable features in the dataset. Moreover, the number of instances for class $a_2$ is highest.

\begin{table}[h]
\centering
\caption{Confusion matrix of [Baseline,$\mathbf{KD}_1$,$\mathbf{KD}_2$,proposed] in $\%$.}
 \resizebox{.48\textwidth}{!}{
\begin{tabular}{p{0.1cm}cc@{}@{}c@{}@{}c@{}@{}c@{}@{}c@{}@{}c}
\hline
\multicolumn{8}{c}{\textbf{Predicted label}}                                                                                                                     \\
\multirow{7}{*}{\textbf{\rotatebox{90}{True label}}} &   & $\mathbf{a_1}$ & $\mathbf{a_2}$ & $\mathbf{a_3}$ & $\mathbf{a_4}$ &$\mathbf{a_5}$ & $\mathbf{a_6}$                 \\
& $\mathbf{a_1}$ & \cellcolor{blue!25}[66,70,73,89] &                   &                   &                   &                   &                   \\
& $\mathbf{a_2}$ &                   & \cellcolor{blue!25}[72,76,84,96] &                   &                   &                   &                   \\
& $\mathbf{a_3}$ &                   &                   & \cellcolor{blue!25}[70,71,80,94] &                   &                   &                   \\
& $\mathbf{a_4}$ &                   &                   &                   & \cellcolor{blue!25}[68,72,78,92] &                   &                   \\
& $\mathbf{a_5}$ &                   &                   &                   &                   & \cellcolor{blue!25}[68,75,79,91] &                   \\
& $\mathbf{a_6}$ &                   &                   &                   &                   &                   & \cellcolor{blue!25}[67,75,82,93]
\\ \hline \hline                                
\end{tabular}
}
\label{rwt1}
\vspace{-0.6cm}
\end{table}

\subsection{Result 3: Accuracy and F$_1$ with unseen class} 
Finally, we study the performance of the proposed scheme, when one class is unseen. In other words, we perform the leave-one-out test in this result. Fig.~\ref{unseen}(a) illustrates the accuracy and F$_1$ score achieved by the proposed scheme, where the instance of the given class is missing from the training dataset, and the ratio of memory consumed ($\alpha$) is $0.65$ . Here, we observe that the accuracy and F$_1$ score decreases when one class is missing. It is because, the built classifier does not hold the features associated with the missing class. Further, the impact of one unseen class varies over another because the number of data instance that generates the most identifiable feature by a classifier changes with the change in the unseen class. Similarly, Fig.~\ref{unseen}(b) illustrates the accuracy and F$_1$ score when instances of the given class are missing from the training dataset and execution time ($\beta$) is $180$ ms.

\begin{figure}[h]
        \centering
\begin{tabular}{cc}
\hspace{-.2cm}\includegraphics[scale=0.460]{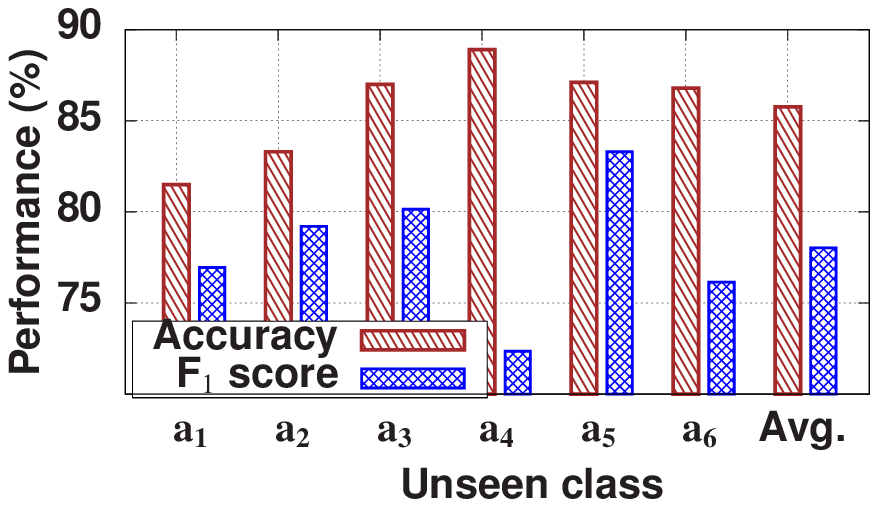}  & \includegraphics[scale=0.460]{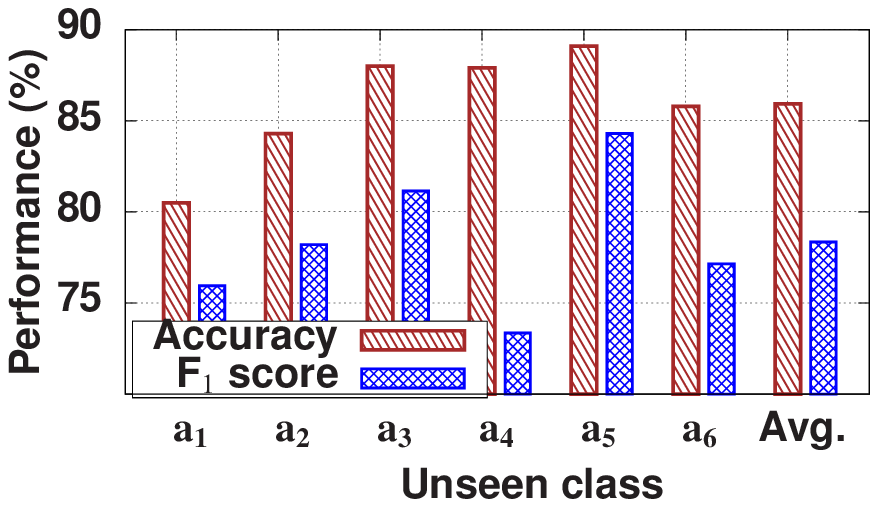}\\
\hspace{0.3cm}\scriptsize{(a) Ratio of memory ($\alpha=0.65$).} & \hspace{0.3cm} \scriptsize{(b) Execution time ($\beta=180$ ms).}
\end{tabular}
        \caption{Accuracy and F$_1$ score with one unseen class.} 
        \label{unseen}
        \vspace{-0.6cm}
\end{figure}

\section{Conclusion}\label{conclude_model}
In this paper, we proposed an approach to design and train a lightweight DNN using a large-size DNN, where trained lightweight DNN satisfied the $\alpha$ and $\beta$ constraints of the edge devices, acronymed as EarlyLight. The approach used optimal dropout selection and factorization for DNN compression. The EarlyLight approach also incorporated knowledge distillation to improve the performance of the lightweight DNN. Further, we introduced an early halting technique to train lightweight DNN, which saved resources; therefore, it speedups the training procedure. We also carried out several experiments to validate the effectiveness of the EarlyLight approach. The results showed that the approach achieved high accuracy on edge devices. This work provides a future direction towards developing a DNN compression technique that can also handle noise in the dataset due to faulty sensors.

\bibliographystyle{IEEEtran}
\bibliography{refer}

%
%

\end{document}